\documentclass[12pt]{article}
\usepackage{amsmath}
\usepackage{graphicx, graphics}
\usepackage{subfigure}
\usepackage{natbib}
\usepackage{url} 

\newcommand{\blind}{0}

\usepackage[ruled,vlined]{algorithm2e}

\usepackage[utf8]{inputenc} 
\usepackage[T1]{fontenc}    
\usepackage{hyperref}       
\usepackage{booktabs}       
\usepackage{amsfonts}       
\usepackage{nicefrac}       
\usepackage{microtype}      
\usepackage{xcolor}         

\usepackage{bbm,bm}
\usepackage{color}
\usepackage{amssymb}
\usepackage{mathtools}
\usepackage{amsthm}
\usepackage{mathrsfs, braket}

\def\E{\mathbb{E}}
\def\N{\mathbb{N}} 
\def\P{\mathbb{P}}

\def\1{\mathbbm{1}}
\def\I{{\bf I}}

\DeclareMathOperator\dif{d\!}
\newcommand{\e}{\mathrm{e}}

\newtheorem{theorem}{Theorem}[section]
\newtheorem{proposition}[theorem]{Proposition}
\newtheorem{lemma}[theorem]{Lemma}
\newtheorem{corollary}[theorem]{Corollary}

\newtheorem{assumption}[theorem]{Assumption}

\newtheorem{remark}[theorem]{Remark}

\usepackage[textsize=tiny]{todonotes}
\begin{document}

\def\spacingset#1{\renewcommand{\baselinestretch}%
{#1}\small\normalsize} \spacingset{1}


\if0\blind
{
  \title{\bf Weighted Averaged Stochastic Gradient Descent: \\
Asymptotic Normality and Optimality}
\author{Ziyang Wei \hspace{.2cm}\\
    Department of Statistics, University of Chicago\\
    Wanrong Zhu \\
    Department of Statistics, University of California, Irvine\\
    and \\
    Wei Biao Wu \\
    Department of Statistics, University of Chicago}
    
  \maketitle
} \fi

\if1\blind
{
  \bigskip
  \bigskip
  \bigskip
  \begin{center}
    {\LARGE\bf Title}
\end{center}
  \medskip
} \fi

\bigskip
\begin{abstract}
Stochastic Gradient Descent (SGD) is one of the most popular algorithms in statistical and machine learning due to its computational and memory efficiency. Various averaging schemes have been proposed to accelerate the convergence of SGD in different settings. In this paper, we explore a general averaging scheme for SGD. Specifically, we establish the asymptotic normality of a broad range of weighted averaged SGD solutions and provide asymptotically valid online inference approaches. Furthermore, we propose an adaptive averaging scheme that exhibits both optimal statistical rate and favorable non-asymptotic convergence, drawing insights from the optimal weight for the linear model in terms of non-asymptotic mean squared error (MSE).
\end{abstract}

\noindent%
\vfill

\newpage
\spacingset{1.5} 
\section{Introduction}
\label{sec:Intr}
Stochastic Gradient Descent (SGD) is one of the most popular optimization algorithms. It (along with its variants) plays a crucial role in statistical and machine learning problems \citep{robbins1951stochastic, lai2003stochastic, bottou2018optimization}. To find the minimizer of a convex function $F$, SGD produces a sequence of iterates through an unbiased estimate of $F$’s gradient/subgradient. The algorithm is easy to implement and popular in applications for its effectiveness, memory and computational efficiency, and online property.

Theoretical properties of SGD have been extensively studied in stochastic approximation theory and machine learning from asymptotic convergence to non-asymptotic analysis \citep{kushner2003stochastic, moulines2011non, toulis2017asymptotic, harvey2019tight}.  After $n$ iterations, optimal convergence rates of optimization error are $\mathcal{O}(1/n)$ and $\mathcal{O}(1/\sqrt{n})$ for strongly-convex and convex problems \citep{nemirovski2009robust, lacoste2012simpler}. It is shown that SGD with a learning rate proportional to the inverse of the number of iterations is not optimal under the wrong setting of the proportionality constant or without strong convexity assumptions. The idea of using \emph{averaging} to accelerate convergence is proposed by \citet{ruppert1988efficient} and \citet{polyak1992acceleration}. They demonstrated that using a learning rate with slower decays, combined with uniform averaging, robustly leads to information-theoretically optimal asymptotic variance. This averaging scheme, where all iterates are averaged, is known as Polyak-Ruppert averaging or averaged SGD (ASGD).   However, it is not optimal from a non-asymptotic perspective \citep{moulines2011non, NIPS2014_f29c21d4}. Moreover, for non-smooth objective functions, neither the final iterate nor ASGD can achieve the optimal convergence rate. To address these issues, various modified versions of ASGD have been proposed, such as suffix averaging \citep{rakhlin2011making} and polynomial-decay averaging \citep{shamir2013stochastic} for non-smooth problems, exponential weighted moving average (EWMA) for capturing time variation, elastic averaging in parallel computing environment \citep{zhang2015deep} and a simple weight proportional to $\mathcal{O}(n)$ for the projected stochastic
subgradient method \citep{lacoste2012simpler}.

As mentioned earlier, employing a suitable averaging scheme in specific settings can help accelerate convergence and necessitates only a straightforward modification to the original SGD algorithm. This paper delves deeper into the characteristics of averaging by examining a comprehensive averaging scheme. Our main objectives include understanding the variability and statistical efficiency of a general weighted averaged SGD. Under certain mild assumptions, we find the optimal condition for the central limit theorem (CLT) of weighted averaged SGD solutions. The condition is optimal in the sense that it is sufficient and cannot be weakened under certain settings. We also establish the asymptotic normality of a generic weighting scheme commonly seen in practice (with slightly stronger yet more easily verifiable conditions on weights). The theoretical results are applicable to a wide range of existing algorithms, including the last iterate of SGD, the polynomial-decay and suffix averaged SGD.  Moreover, statistical inference for weighted averaged SGD is a valuable byproduct owing to its close relation to ASGD. We refine the functional CLT of SGD originally presented in \cite{lee2022fast} with much weaker assumptions. This means we can effectively utilize inference methods, such as covariance matrix estimation \citep{chen2020statistical, zhu2021online} or asymptotic pivotal statistics \citep{lee2022fast, zhu2024high}, originally designed for ASGD in this context. 
Beyond asymptotic normality, we also investigate finite sample convergence. When considering finite sample MSE, it is challenging to identify a single averaging scheme that is optimal for all objective functions. To gain insights, we examine the linear model to derive adaptively weighted SGD iterates that minimize the finite sample MSE. Our numerical study indicates that the adaptive weight derived from the linear model not only achieves the optimal statistical rate but also exhibits favorable non-asymptotic convergence on many models.

\noindent{\textbf{Related Works.}} There has recently been increasing interest in the statistical inference of SGD and other stochastic approximation procedures. One line of the work is to study the limiting distribution. Following the celebrated asymptotic normality results of ASGD \citep{ruppert1988efficient, polyak1992acceleration}, similar asymptotic normality results are proposed for specific algorithms in different settings such as non-convex settings, linear systems, and online decision-making \citep{yu2021non-convex, mou2020linear, chen2021onine}. Also, various methods based on covariance matrix estimation or bootstrap/subsampling are proposed to construct asymptotically valid confidence intervals \citep{chen2020statistical,zhu2021online, li2018statistical, Fang18boot, liang2019statistical, lee2022fast}.
In addition to asymptotic results, there are also many non-asymptotic works related to inference, for example, non-asymptotic normal approximation of SGD \citep{anastasiou19a} and concentration analysis \citep{davis2021low, harvey2019tight, lou2022beyond}.



The remainder of this paper is organized as follows. In Section \ref{sec2}, we introduce the main theorem for asymptotic normality along with statistical inference methods that can be implemented in an online fashion. In Section \ref{sec3}, we present two significant examples: polynomial-decay averaging and suffix averaging, both of which are applicable to our theorem. Additionally, we adapt the suffix averaging process for online implementation. In Section \ref{sec4}, we identify optimal weights for the linear model and propose a novel adaptive weighting scheme that achieves both optimal statistical rates and favorable finite sample performance. Numerical results are provided in Section \ref{sec5}. Finally, Section \ref{sec6} encompasses a discussion and outlines potential future work.

\section{Asymptotic Normality for Weighted SGD} 
\label{sec2}
\subsection{Overview}
In many statistical estimation and machine learning problems, we need to estimate the minimizer of a convex objective function $F(x)$, mapping from $\mathbb{R}^d$ to $\mathbb{R}$, as $F(x)=\E_{\xi \sim \Pi}f(x,\xi),$ where $f(x,\xi)$ is a loss function and $\xi$ is a random variable following the distribution $\Pi$. Assuming the minimizer $$x^{*}=\mathop{\arg\min}_{x \in \mathbb{R}^d } F(x)$$ exists, the SGD updates the iterate as follows (initialized at $x_0$),
\begin{equation}\label{eq:1}
	x_i=x_{i-1}-\eta_i\nabla f(x_{i-1},\xi_i), \ i\ge 1,
\end{equation}
where $\xi_{i}$ are $\emph{i.i.d}$ from $\Pi$, $\eta_i$ is the $i$-th step size or learning rate, $\nabla f(x,\xi)$ is the gradient of $f$ with respect to the first argument, and $x_0$ is a starting point. Let $\bar{x}_n = \sum_{i=1}^{n}x_{i} / n$ be the uniform average. Under suitable conditions,  \citet{polyak1992acceleration} shows that
\begin{equation}\label{eq:pjclt}
\sqrt{n}(\bar{x}_n-x^*) \stackrel{D}{\to} \mathcal{N}(0, V), \mbox{ where } V = A^{-1}SA^{-1},
	A=\nabla^2F(x^*), S=\mathbb{E}([\nabla f(x^*,\xi)][\nabla f(x^*,\xi)]^T).
\end{equation}
In this paper, one of our goals is to establish the asymptotic normality for the general weighted average 
\begin{equation}\label{wasgd}
	\tilde{x}_n=\sum_{i=1}^n w_{n,i} x_i,
\end{equation}
where $w_{n,i}$, $1\leq i \leq n$, denotes the weight of $x_i$ after the $n$-th update with $\sum_{i=1}^n w_{n,i} = 1$. We propose a general condition on $\{ w_{n,i}\}_{1\leq i \leq n}$ for the weighted averaged SGD \eqref{wasgd} to be asymptotically normal. Additionally, we also provide easily verifiable conditions for the asymptotic normality result with a modified weighting condition. To be specific, we show that under certain mild assumptions,
	\begin{equation*}
		\sqrt{n}(\tilde{x}_n-x^*) \stackrel{D}{\to} \mathcal{N}(0,w V), \mbox{ where }
		w = \lim_{n \rightarrow \infty}n\sum_{i=1}^{n}(w_{n,i})^{2}
	\end{equation*}
is a prefactor caused by weighting, and $V$ is defined in \eqref{eq:pjclt}. We demonstrate that this result holds for many existing averaging schemes, as well as the adaptive averaging we present in Section \ref{newave}.

We now introduce some notation. Throughout the paper, for $t\in \mathbb{R}$, $\lfloor t \rfloor = \max\{i \in \mathbb Z: i \le t\}$ and $\lceil t \rceil = \min\{i \in \mathbb Z: i \ge t\}$. We use $\mathscr{F}_{i}$ to denote the nested $\sigma$-algebra generated by $\{\xi_1, \ldots , \xi_i\}$. Let $\mathbb{E}_i$ denote the conditional expectation $\mathbb{E}(\cdot|\mathscr{F}_i)$, and $\mathbb{P}_i$ denote the conditional probability $\mathbb{P}(\cdot|\mathscr{F}_i)$. We use $\stackrel{D}{\to}$ to denote convergence in distribution, and $\stackrel{P}{\to}$ to denote convergence in probability. For a vector $a = (a_1, \ldots, a_p)^\top$ let the norm $| a | = (\sum_{i=1}^p a_i^2 )^{1/2}$. Define $|A|_F$, $|A|$, $\lambda_{\text{max}}(A)$, and $\lambda_{\text{min}}(A)$ as the Frobenius norm, operator norm, the largest eigenvalue, and the smallest eigenvalue of a matrix $A$. For notation simplicity we use $C$, $\tilde{C}$, $K$ and $C_1$, $C_2, \ldots$, to denote different constants, whose values may change in different equations.

\subsection{Main Results}
We begin by introducing several assumptions. For a random variable $X$, we write $X \in {\cal L}^p$, $p > 0$, if $\| X \|_p : = (\E |X|^p )^{1/p} < \infty$, and write $\| X \| = \| X \|_2$.

\begin{assumption} \label{as1}
$F(x)$ is continuously differentiable and strongly convex with parameter $\mu>0$. That is, for any $x_1$ and $x_2$, the inner product $\Braket{\nabla F(x_2) - \nabla F(x_1), x_2-x_1} \ge {\mu} |x_1-x_2 |^2$, or equivalently
$$F(x_2)\ge F(x_1)+\Braket{\nabla F(x_1) ,x_2-x_1}+\frac{\mu}{2} |x_1-x_2 |^2.$$
\end{assumption} 
\begin{assumption}\label{as2}
The function $f(x,\xi)$ is continuously differentiable w.r.t. $x$ for any $\xi$, $\sigma(x) := \|\nabla f(x,\xi)\| < \infty$, and $\nabla f(x,\xi)$ is stochastic Lipschitz continuous with parameter $L_2$, i.e., for any $x_1$ and $x_2$,
	$$ \| \nabla f(x_1,\xi)-\nabla f(x_2,\xi) \| :=
	(\E |\nabla f(x_1,\xi)-\nabla f(x_2,\xi)|^2)^{\frac{1}{2}}\leq L_2 |x_1-x_2|_2.$$
\end{assumption}

Assumption \ref{as1} requires strong convexity and the existence of Hessian at the true parameter, which are needed to derive the asymptotic normality of ASGD solutions. These properties are also important for obtaining the desired error bounds on SGD iterates and the asymptotic properties of the weighted averaged SGD. Assumption \ref{as2} ensures that Leibniz's integration rule holds. Consequently, the gradient noise $\epsilon_i= \nabla F(x_{i-1})-\nabla f(x_{i-1},\xi_i)$ is a martingale difference, i.e., $\mathbb{E}_{n-1} \epsilon_n = \mathbb{E}( \epsilon_n |\mathscr{F}_{n-1}) =0$. The stochastic Lipschitz continuity condition of the gradient guarantees the $L$-smoothness of $F$ and the boundedness of gradient noise moments. Easily verified examples include linear and logistic regression. Similar assumptions have been adopted in the literature on asymptotics and statistical inference of SGD \citep{polyak1992acceleration, chen2020statistical,li2022root,zhu2021online,listochastic}.

With the aforementioned assumptions in place, we will now present our main results regarding the very general condition of the asymptotic normality for weighted averaged SGD solutions. Theorem \ref{tm0} allows general weights $w_{n, i}$, and it is quenched in the sense that it can allow any starting point $x_0$. Recall (\ref{eq:pjclt}) for $A$ and $S$. We first introduce the expression of the asymptotic covariance matrix of $\tilde{x}_n$:
\begin{equation}\label{asymcov}
	V_n=\sum_{i=1}^n \Phi_{n,i}S\Phi_{n,i} \mbox{ where } \Phi_{n,i} = \eta_i \sum_{k=i}^n w_{n,k} \prod_{j=i+1}^k ({\textbf{I}}_d-\eta_jA).     
\end{equation}

\begin{theorem}\label{tm0}
Given SGD \eqref{eq:1} with step size $\eta_i=\eta i^{-\alpha}$ for some $\eta>0$ and $1/2<\alpha<1$, we consider the general averaging scheme (\ref{wasgd}) with $w_{n,i} \ge 0$. (i) Under Assumptions \ref{as1} and \ref{as2}, the following condition
\begin{equation}\label{cond1}
\max_{1\leq i \leq n} |V_n^{-\frac{1}{2}} \Phi_{n,i} | \rightarrow 0
\end{equation}
is sufficient for the quenched central limit theorem: for any starting point $x_0$,
	\begin{equation}
		\label{cond2}
		V_n^{-1/2}(\tilde{x}_n-x^*)\stackrel{D}{\to} \mathcal{N}(0,{\textbf{I}}_d).
	\end{equation}
(ii) Consider the special case where \( f(x, \xi) = \frac{1}{2} | \xi - A^{\frac{1}{2}} x |^2 \) and assume that \( \nabla f(x^*, \xi_i) \) is not normally distributed. Then condition (\ref{cond1}) is also \textit{necessary} for the CLT (\ref{cond2}).   
\end{theorem}

Theorem \ref{tm0}(ii) represents the first known necessary condition for the CLT to hold for SGD solutions. The proof relies on the deep Lévy-Cramér theorem, which characterizes the factor closeness of Gaussian distributions.

\begin{remark}
In the appendix Section \ref{sec:general} we provide a more general version of quenched CLT, which further allows a general schedule of the learning rates $\eta_i$. 
\end{remark}

By setting $w_{n,n}=1$ and $w_{n,i}=0$ for $1\leq i <n$, Theorem \ref{tm0} implies the following Corollary \ref{CLTlast}, which asserts CLT for the last iterate of SGD when $1/2 < \alpha < 1$. To this end we shall show the convergence of the covariance matrix. Details can be found in the appendix. Our imposed conditions appear quite mild. Similar CLTs under restrictive conditions are obtained in \cite{Chung54}. The special case with $\alpha = 1$ is considered in \cite{sacks1958asymptotic, fabian1968asymptotic, McLeish1976, borkar2008stochastic,toulis2017asymptotic}, among others. In other words, we provide a novel and self-contained proof of the CLT for the last iterate under mild conditions.  

\begin{corollary}\label{CLTlast}
Under Assumptions \ref{as1} and \ref{as2}, the last SGD iterate in \eqref{eq:1} with step size $\eta_i=\eta i^{-\alpha}$ for some $\eta>0$ and $1/2 < \alpha<1$ is asymptotically normal:
$$n^{\frac{\alpha}{2}}(x_n-x^*) \stackrel{D}{\to} \mathcal{N}(0,\eta {V_0}),$$
where $V_0$ is the unique solution to the Lyapunov equation $AV_0 + V_0A = S.$
\end{corollary}

Since it is usually implicit whether a specific weighting scheme meets condition \eqref{cond1}, we introduce in Theorem \ref{tm2} and Remark \ref{reducecond} two more explicit ways to ensure the CLT for a general weighted averaging scheme.

\begin{theorem}\label{tm2}
Given SGD iterates in \eqref{eq:1} with step size $\eta_i=\eta i^{-\alpha}$ for some $\eta>0$ and $1/2<\alpha<1$, we consider the general averaging scheme (\ref{wasgd}) where the weight $w_{n,i}$ satisfies the following conditions:
\begin{enumerate}
\item $|w_{n,i}| \leq C n^{-1}$ for some constant $C$.
\item  $w = \lim_{n \rightarrow \infty}n\sum_{i=1}^{n}(w_{n,i})^{2}$ exists.
\item piecewise Lipshitz: there exist functions $\{f_n\}$ from $[0,1]$ to $\mathbb{R}$, a finite set $\{t_1,t_2,...,t_N\} \subseteq [0,1]$ and a constant $L_1>0$ such that $w_{n,k}=f_n(k/n)/(\sum_{k=1}^nf_n(k/n))$ with $\sum_{k=1}^nf_n(k/n) \asymp n$ and 
    $|f_n(s_1)-f_n(s_2)| \leq L_1|s_1-s_2|,$
    for all $n \in \mathbb{N}^+$ and all $s_1$, $s_2 \in [t_j,t_{j+1}]$ where $1\leq j \leq N-1$.
\end{enumerate}
Then under Assumptions \ref{as1} and \ref{as2}, we have the CLT $\sqrt{n}(\tilde{x}_n-x^*) \stackrel{D}{\to} \mathcal{N}(0,w V)$, where $V$ is defined in \eqref{eq:pjclt}
\end{theorem}

\begin{remark}
The asymptotic covariance of the weighted average $\tilde{x}_n$ here is composed of a prefactor  $w$ and the sandwich form $A^{-1}SA^{-1}$. In the context of ASGD, where $w_{n,i} = 1/n$, the prefactor $w = 1$, which aligns with our results. We also show in the appendix that, this asymptotic covariance is exactly the limit of $nV_n=n\Phi_{n,i}S\Phi_{n,i}$ as defined in Theorem \ref{tm0}, indicating the consistency of our main results. 
\end{remark}

\begin{remark}\label{reducecond}
The piecewise Lipshitz condition requires that the majority of weights should not undergo drastic changes. This mild condition is satisfied by a variety of averaging schemes (see Sections \ref{sec3} and \ref{sec4}). A slightly stronger yet simpler condition is $|w_{n,i+1}-w_{n,i}| \leq \tilde{C} n^{-2}$ for some constant $\tilde{C}>0$.
\end{remark}

\subsection{Online Statistical Inference}
The asymptotic normality result presented in Theorem \ref{tm2} enables statistical inference for the weighted averaged SGD, which includes tasks such as constructing confidence intervals. We will explore two online methods, which, although originally designed for ASGD, are applicable in this context as well.

\noindent{\textbf{Via covariance matrix estimation.}}
One approach to construct confidence intervals involves estimating the limiting covariance matrix $\Sigma$ and directly utilizing the asymptotic normality result outlined in Theorem \ref{tm2}. For example, the $95\%$ confidence interval for the $j$-th element $x^{*}_{j}$ can be constructed as:
$$(\tilde{x}_{n})_{j} + z_{0.975}\sqrt{w(\hat{V}_{n})_{jj}/n},$$
where  $z_{0.975}$ corresponds to the $97.5\%$ percentile for a standard normal distribution, $\hat{V}_{n}$ is the estimate for the sandwich form matrix $V = A^{-1}SA^{-1}.$ The prefactor $w$ can be readily computed for any specific averaging scheme, meaning that we only need to estimate the sandwich form matrix  $V$. It is important to note that this sandwich form matrix $V$ is the asymptotic covariance matrix of ASGD $\bar{x}_{n}$. There are several established approaches to estimating the asymptotic covariance matrix of ASGD in an online fashion, such as the plug-in method, which estimates $A$ and $S$ separately, and the batch-means method, which employs SGD iterates to directly construct $\hat{V}_{n}$ \citep{chen2020statistical,zhu2021online}. As we are utilizing the vanilla SGD algorithm, these aforementioned approaches can be applied to obtain a consistent sandwich form matrix estimate $\hat{V}_{n}$.
Then, we can construct the asymptotic valid confidence intervals. 

\noindent{\textbf{Via asymptotically pivotal statistics.}}
An alternative approach to constructing confidence intervals is to use asymptotically pivotal statistics. One method involves random scaling, where we studentize $\sqrt{n}(\tilde{x}_n-x^*)$ using the  random scaling matrix $\hat V_{rs,n}$ as described in \cite{lee2022fast}, i.e., 
$$
\hat V_{rs, n} = \frac{1}{n}\sum_{s=1}^{n}\left\{\frac{1}{\sqrt{n}}\sum_{i=1}^{s}(x_{i} - \bar{x}_{n})\right\}\left\{\frac{1}{\sqrt{n}}\sum_{i=1}^{s}(x_{i} - \bar{x}_{n})\right\}^{T}.
$$
The following functional CLT in Theorem \ref{th:fclt} concerns weak convergence for the partial sum process $S_j = \sum_{i=1}^j (x_i - x^*)$. The imposed conditions are quite mild and they are easily verifiable. The moment condition in Theorem \ref{th:fclt} is nearly optimal since it only requires $(2+\epsilon)$th moment. Our functional CLT substantially relaxes the early similar result in \cite{lee2022fast} which requires a much stronger moment condition with $p \ge 2/(1-\alpha)$. By the continuous mapping theorem, (\ref{pivotal}) follows from (\ref{fclt}). 

\begin{theorem}
\label{th:fclt} Assume $\E (|\nabla f(x^*,\xi)|^p) < \infty$ for some $p > 2$. Let $I\! B(u) = (I\! B_1(u), \ldots, I\! B_d(u))$ be the standard $d$-dimensional Brownian motion. Under Assumptions \ref{as1} and \ref{as2}, we have the functional CLT:
\begin{equation}\label{fclt}
\{ S_j/\sqrt n, 1 \le j \le n\} \Rightarrow V^{1/2} \{ I\! B(u), 0 \le u \le 1 \}.
\end{equation}
Consequently we have the asymptotic pivotal convergence
\begin{equation}\label{pivotal}
 {\sqrt{n}((\tilde{x}_n)_{j}-x^*_{j})}/{\sqrt{(\hat V_{rs,n})_{jj}}}\stackrel{D}{\to}
 U\, \mbox{ where } U = I\! B_1(1)/{\sqrt{\int_{0}^{1}\{I\! B_1(r) - r I\! B_1(1)\}^2dr}}.   
\end{equation}
\end{theorem}

With (\ref{pivotal}), one can construct confidence intervals based on Theorem \ref{th:fclt}. For instance, the $95\%$ confidence interval for the $j$-th element $x^{*}_{j}$ can be constructed as: $(\tilde{x}_{n})_{j} \pm u_{rs, 0.975} \{ (\hat V_{rs,n})_{jj} / n \}^{1/2},$ where $u_{rs, 0.975}$ is the $97.5\%$ quantile for $U$. Note that the random scaling matrix $\hat V_{rs, n}$ can be updated recursively, thereby enabling the construction of confidence intervals in an online fashion. Besides the random calling method, one can also employ the parallel run method to construct a $t$-type statistic using a small number of parallel runs of weighted averaged SGD; we refer to \cite{zhu2024high} for details.

\section{Examples}\label{sec3}
In this section, we shall apply our results to two widely used specific examples of averaging schemes: polynomial-decay averaging \citep{shamir2013stochastic} and suffix averaging \citep{rakhlin2011making}. The original suffix averaging scheme is not recursive. Here we shall modify it to an online fashion in which estimates can be computed recursively. We will show that they are not statistically optimal, with a constant prefactor that is strictly greater than unity; see Corollaries \ref{poly} and \ref{Co2} below.


\subsection{Polynomial-decay Averaging}
With a small number $\gamma \ge 1$, the polynomial-decay averaging \citep{shamir2013stochastic} is defined as follows: Given iterates $\{x_{i}\}_{i=1}^{\infty}$, $\tilde{x}_1 = x_{1}$, and for any $n\ge 1$,
\begin{equation}\label{eq:2}
	\tilde{x}_n=(1-\frac{\gamma+1}{\gamma+n})\tilde{x}_{n-1}+\frac{\gamma+1}{\gamma+n}x_n.
\end{equation}
The recursion form in \eqref{eq:2} can be rewritten as the weighted average $\tilde{x}_n=\sum_{i=1}^n w_{n,i} x_i$ with weights
\begin{equation*}
w_{n,i} = \frac{\gamma+1}{\gamma+i} \prod_{j=i+1}^{n} \frac{j-1}{j+\gamma}  =\frac{\gamma+1}{n}\frac{\Gamma(\gamma+i+1)\Gamma(n+1)}{\Gamma(\gamma+n+1)\Gamma(i+1)}, \mbox{ where }
\Gamma(x)=\int_0^{\infty} t^{x-1}e^{-t}dt.
\end{equation*}
The weights $w_{n,i}$ satisfy conditions in Theorem \ref{tm2}. Therefore we have the following CLT in view of 
$$\lim_{n \rightarrow \infty} n\sum_{i=1}^n (w_{n,i})^2=\frac{(\gamma+1)^2}{2\gamma+1} =: \tau.$$ 

\begin{corollary}\label{poly}
Consider SGD \eqref{eq:1} with step size $\eta_i=\eta i^{-\alpha}$ for $\eta>0$, $0.5<\alpha<1$, and polynomial-decay averaging $\tilde{x}_n$ defined in \eqref{eq:2}. Under Assumptions \ref{as1} and \ref{as2}, we have
$\sqrt{n}(\tilde{x}_n-x^*) \stackrel{D}{\to} \mathcal{N}\left(0,\tau V\right).$
\end{corollary}

\subsection{Suffix Averaging}
The $\kappa$-suffix averaging in \cite{rakhlin2011making} is defined as the average of the last $\lceil \kappa n \rceil$ iterates:
\begin{equation}\label{suffix}
\tilde{x}_n=\frac{1}{\lceil \kappa n \rceil} \sum_{i=\lceil(1-\kappa)n\rceil}^n x_i, \,
  0<\kappa<1
\end{equation}
For $\kappa$-suffix averaging,  the weight $w_{n,i} = 1/\lceil \kappa n \rceil$ for $i>(1-\kappa)n$ otherwise $0$. The weight satisfies conditions in Theorem \ref{tm2} with $\lim_{n \rightarrow \infty}n\sum_{i=1}^{n}(w_{n,i})^{2} = 1/\kappa.$ Thus we have the following CLT:

\begin{corollary}\label{Co2}
Consider SGD iterates in \eqref{eq:1} with step size $\eta_i=\eta i^{-\alpha}$ for $\eta>0$, $0.5<\alpha<1$, and $\tilde{x}_n$ defined in \eqref{suffix}. Under Assumptions \ref{as1} and \ref{as2}, we have
$\sqrt{n}(\tilde{x}_n-x^*) \stackrel{D}{\to} \mathcal{N}(0, {\kappa}^{-1} V).$
\end{corollary} 

\begin{remark}[Online algorithm for suffix averaging]
Since $(1-\kappa)n$ depends on $n$, the $\kappa$-suffix averaging cannot be computed on-the-fly unless all iterations are stored and the stopping time $n$ is known beforehand. To enable “on-the-fly” computation, we modify the suffix averaging procedure to an online method. We employ the concept of online batch scheme: divide the rounds into blocks and track iterations within the current block (or the most recent blocks). The block sizes are pre-defined based on various objectives and training parameters. For a pre-defined sequence $(a_{m})_{m\ge 0}$, we treat $x_{a_{m}}$ as the start of the $m$-th block. Let $m_{t}$ denote the block index for the $t$-th iteration, satisfying $a_{m_t}\le t < a_{m_t+1}$. 

In the online suffix averaging procedure, we partition the rounds into exponentially increasing blocks, and maintain the average of the last two blocks;
see Figure \ref{fig:wxample} for an example of possible realizations. In particular, we set $a_{m} = \lfloor 2^{m-1} \rfloor +1, m\ge 0$. Then  $B_{0} = \{x_{1}\}, B_{1} = \{x_{2}\}, B_{2} = \{x_{3}, x_{4}\}, B_{3} = \{x_{5}, x_{6}, x_{7}, x_{8}\}, ...$,  
and the end index of the $m$-th block is $2^{m}$. We have $m_t =\lceil \log_{2}t\rceil$, i.e., $\lfloor 2^{m_t -1} \rfloor < t \le 2^{m_t}$. Given the SGD iterates $x_{1}, x_{2}, \ldots$, the online suffix averaging procedure is defined as follows: $\hat x_{1} = x_{1}, \hat x_{2} = (x_{1}+x_{2})/2 $, 
\begin{equation}\label{eq:online_suffix}  
	\hat x_{t}=\frac{1}{t - 2^{ \lceil \log_{2}t\rceil-2}}\left(\sum_{k=2^{\lceil \log_{2}t\rceil-2}+1}^{2^{\lceil \log_{2}t\rceil-1}}x_{k} + \sum_{k=2^{\lceil \log_{2}t\rceil-1}+1}^{t}x_{k}\right), \,\, t \ge 3.
\end{equation}
Note that $1/2 < (t - 2^{\lceil \log_{2}t\rceil-2})/t\le 3/4$ for $t\ge 3$. 
Therefore, the online suffix averaging is a form of robust suffix averaging with $1/2 <\kappa\le 3/4$, i.e., the average would always correspond to a constant-portion suffix of all iterates. The online suffix average $\hat x_{t}$ in \eqref{eq:online_suffix} can be updated recursively; see Algorithm \ref{alg}.

\begin{figure} 
	\centering
	\includegraphics[width=0.7\textwidth]{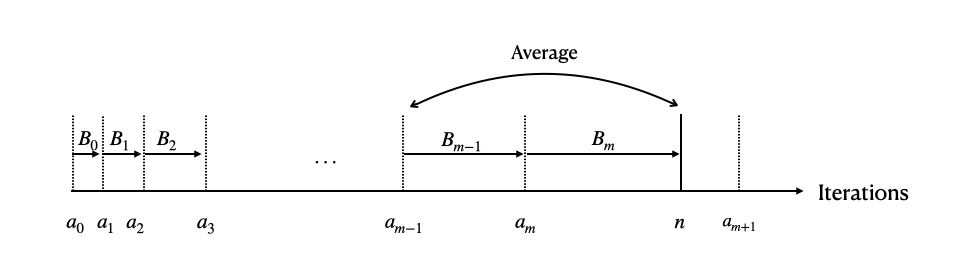}
	\caption{Realizations of online suffix averaging. Here $a_{m}, m\ge 0$, is the index of the starting point of the $m$-th block.}
	\label{fig:wxample} 
\end{figure}

\begin{algorithm}[h]
\caption{Online suffix averaging}\label{alg}
	\SetAlgoLined
    \hrule
	{\textbf{Input}}: step sizes $\{\eta_{t}\}_{t\ge 1}$, initialization 
	$x_{0}, m = 0, S_{0} = 0, S_{1} = 0$\\
	\For{$t= 1,2, ..., $}{
		$x_{t} = x_{t-1} -\eta_{t}\hat{g}_{t}$ ($\hat{g}_{t}$ is stochastic gradient)  \\
		\eIf{$t > 2^{m}$}{
			$m = m+1$, $S_{0} = S_{1}$, $S_{1} = x_{t}$\;
		}{ 
		$S_{1} = S_{1} + x_{t}$\;
		}
	{\textbf{Output}} (if necessary): $\hat x_{t} = (S_{0} + S_{1})/(t - \lfloor 2^{m-2}\rfloor)$  
	} 
    \hrule
\end{algorithm}
\end{remark}

\section{Non-asymptotic Mean Squared Error (MSE)}
\label{sec4}
In addition to the asymptotic distribution and statistical convergence rates, it is also important to consider finite sample performance when dealing with finite data problems or when early stopping is desired. In this section, we will examine the optimal weight for a linear model in terms of finite sample MSE, building upon the concept of \emph{best linear unbiased estimation} (BLUE). Furthermore, we will introduce a novel adaptive averaging scheme based on insights from the mean estimation model. This particular scheme is both statistically efficient with optimal variance, and has a fast finite sample convergence rate, outperforming existing averaging schemes in the mean estimation model. For (\ref{wasgd}), we can evaluate its non-asymptotic performance through its $\text{MSE}(\tilde{x}_n) = \E|\tilde{x}_{n} - x^{*}|^2.$ To obtain the MSE-optimal weights, we consider
\begin{equation}\label{eq:optMSE}
	\min_{c = (c_1,\ \cdots,\ c_n): \, c^{T}\1 = 1}\mathbb{E}|\sum_{i=1}^n c_ix_i-x^{*}|^2.
\end{equation}
The solution to the above constrained optimization problem is 
\begin{equation}\label{matrixform}
	c = \frac{\Sigma^{-1}\1}{\1^{T}\Sigma^{-1}\1}, \mbox{ where }
    \Sigma = (\E((x_{i} - x^{*})^{T}(x_{j} - x^{*})))_{1\le i,j \le n}. 
\end{equation}
The solution depends on the correlation between SGD iterates and can vary across different models. In this section, we shall study the linear regression model which can have a close-form solution of the BLUE weights. The latter can provide insights into the properties of an optimal weight in a generalized form. 

\subsection{Linear model}
Consider the following linear regression model:
\begin{equation}\label{eq:linearmodel}
	b_i = a_ix^* + \epsilon_{i},
\end{equation} 
where $x^*$ denotes the unknown parameter of interest, $\epsilon_{i}$ are \emph{i.i.d.}  standard normal, and $\xi_i = (a_{i}, b_{i})$ denote the observed streaming data. To solve the above linear regression problem, we consider the squared loss function $F(x) = \E(f(x,\xi_i))=\E (a_i x-b_i)^2/2,$ and SGD sequence with step size $\eta_{i}$ at the $i$-th iteration:
\begin{equation}\label{onedl}
	x_i=x_{i-1}-\eta_ia_i(a_ix_{i-1}-b_i).
\end{equation}

\begin{proposition}\label{1dlm0}
Consider the linear model in \eqref{eq:linearmodel} and SGD sequence $x_i$ defined in \eqref{onedl} with step size $\eta_1=a_1^{-2}$ and general $\eta_{i}, i\ge 2$. 
The unique solution to the optimization problem  \eqref{eq:optMSE}  is given by
$$c_{n,i}=\frac{a_{i+1}^2+ \eta_i^{-1}-\eta_{i+1}^{-1}}{S_n}, 1\leq i \leq n-1, \, c_{n,n}=\frac{1}{\eta_nS_n}, \mbox{ where } S_n=\sum_{i=1}^n a_i^2.$$
\end{proposition}

The weights in Proposition \ref{1dlm0} are adjusted by the learning rates $\eta_{i}$. However, one characteristic of these weights is that the last weight is significantly larger than the preceding ones. In comparison with the equal weights scheme, our $c_{n, i}$ is intuitively attractive by putting most weight on the last iterate.

\subsection{A new averaging scheme: adaptive weighted averaging}
\label{newave}
    In the special case of the mean estimation model, where $a_{i} = 1, \forall i\ge 1$, 
    then the weights are given by:
\begin{equation}\label{eq:optimalweight}	 
c_{n,i}=\frac{1+\eta_i^{-1}-\eta_{i+1}^{-1}}{n}, \, 1\leq i \leq n-1, \, c_{n,n}=\frac{1}{n\eta_{n}}.
\end{equation}
When the number of iteration increases from $n$ to $n+1$, we have $c_{n+1,i}=(n+1)^{-1} n c_{n,i}, 1\leq i\leq n-1.$ Therefore the weighted averaged SGD $\tilde{x}_n$ with above optimal weights can also be recursively updated: 
\begin{equation}\label{eq:optimal_recursion}
		\tilde{x}_{n+1}=\frac{n}{n+1}\tilde{x}_{n}+\frac{1-\eta_{n+1}^{-1}}{n+1}x_n+\frac{\eta_{n+1}^{-1}}{n+1}x_{n+1}.
	\end{equation}
	The newly introduced averaging scheme, referred to as \emph{adaptive weighted averaging}, effectively decreases the weight of earlier iterates in comparison to the newly arriving ones. From the definition of the optimal weights, the weights in \eqref{eq:optimalweight} minimizes the finite sample MSE among all possible weights for the mean estimation model.
On the other hand, it also satisfies the condition in Theorem \ref{tm0}, which indicates that the corresponding weighted averaged SGD satisfies CLT with the asymptotic covariance matrix being the same as that of ASGD estimates; see the following Corollary \ref{Co3}. Thus, the proposed adaptive weighted average achieves both fast finite sample convergence rates and the optimal statistical rate.
	

	\begin{corollary}\label{Co3}
		Under the settings in Theorem \ref{tm0}, for weights $c_{n,i}$ in \eqref{eq:optimalweight}, we have 
		$\sqrt{n}(\tilde{x}_n-x^*) \stackrel{D}{\to} \mathcal{N}(0,V).$
	\end{corollary}

	\noindent{\textbf{Connection with uniform averaging.}} It is interesting to compare the adaptive weighted average in \eqref{eq:optimal_recursion} with uniform average (ASGD). The recursion of ASGD $\bar{x}_{n}$ takes the following form 
	\begin{equation}\label{eq:asgd}
		\bar x_{n+1} = \frac{n}{n+1}\bar x_{n} + \frac{1}{n+1}x_{n+1}.
	\end{equation} 
	To build the connection between \eqref{eq:optimal_recursion} and \eqref{eq:asgd}, we rewrite \eqref{eq:optimal_recursion} as
\begin{equation}\label{eq:optimal_recursion2} 
		\tilde{x}_{n+1}=\frac{n}{n+1}\tilde x_{n} + \frac{1}{n+1}x_{n+1} 
		+ \frac{\eta_{n+1}^{-1}-1}{n+1}(x_{n+1} - x_{n}). 
	\end{equation}
	Thus we can consider the proposed adaptive weighted average as a modified ASGD with a correction term, where the correction term reduces the weight of earlier iterates and increases the weight of the latest iteration. This modification bears a certain similarity with other existing variance-reduced modifications on SGD where a correction term is applied on stochastic gradients, such as SGD with momentum \citep{Nesterov1983AMF,kingma2014adam, cutkosky2019momentum}.

	\section{Numerical Experiment} 
	\label{sec5}
	In this section, we check the asymptotic normality property of the general weighted averaged SGD and investigate the non-asymptotic performance of various averaging schemes in different settings.

\subsection{Asymptotic normality for different averaging schemes}\label{sec:simu1}
To verify the asymptotic normality and the limiting covariance matrix derived in Theorem \ref{tm2}, we consider three averaging schemes: polynomial-decay, suffix averaging (as described in Section \ref{sec3}), and the adaptive averaging scheme proposed in \eqref{eq:optimal_recursion}. We focus on two classes of loss functions: squared loss $f(x,\xi_i  = (a_{i}, b_{i}))=(a_i^T x-b_i)^2/2$ for the linear regression model, and logit loss: $f(x,\xi_i = (a_{i}, b_{i}))=\log (1+\exp (-b_ia^T_ix))$ for the logistic regression model. In both models,  we assume that the data $\xi_i=(a_i,b_i)$, are independent, where $a_i$ represents the explanatory variable generated from $\mathcal{N}(0,\I_d)$, and $b_i$ represents the response variable generated from two different distributions correspondingly. For linear regression, we assume $b_i \sim \mathcal{N}(a_i^Tx^*,1)$, while for logistic regression, $b_i \in \{1,-1\}$ is generated from a Bernoulli distribution, where $\mathbb{P}(b_i|a_i)=1/(1+\exp (-b_ia^T_ix^*))$. For asymptotic normality we are going to verify in Theorem \ref{tm2}. For squared loss, it is easy to derive that $A=S=\I_d$ and therefore $V = \I_d$. For logit loss, 
	since the explicit forms for $A$ and $S$ are difficult to obtain, we use Monte-Carlo simulation to numerically compute the sandwich form matrix $V$.

	In simulations, we set $d=5$ and the true parameter $x^*=(1,-2,0,0,4)^T$ for both models. We generate SGD sequences with $\eta_{i}=i^{-\alpha}, \alpha=0.505$, and apply different averaging schemes. For the polynomial-decay averaging, we choose $\gamma=3$ \cite{shamir2013stochastic}, and for the suffix averaging we choose $\kappa=0.5$ \cite{rakhlin2011making}. Then the prefactors $w$ for polynomial-decay and suffix averaging schemes are $16/7$ and $2$. 
	For polynomial decay and suffix averaged SGD, we plot the density of the standardized error with and without prefactor $w$, i.e., $w^{-1}V^{-1}\sqrt{n}(\tilde{x}_n-x^*)$ and $V^{-1}\sqrt{n}(\tilde{x}_n-x^*)$. For adaptive weighted SGD, the prefactor is $1$ according to Corollary \ref{Co3}, so we only plot the  density of the standardized error $V^{-1}\sqrt{n}(\tilde{x}_n-x^*)$. 
	
	\begin{figure}[t] 
		\centering
		\subfigure[Squared Loss]{\includegraphics[width=0.45\textwidth]{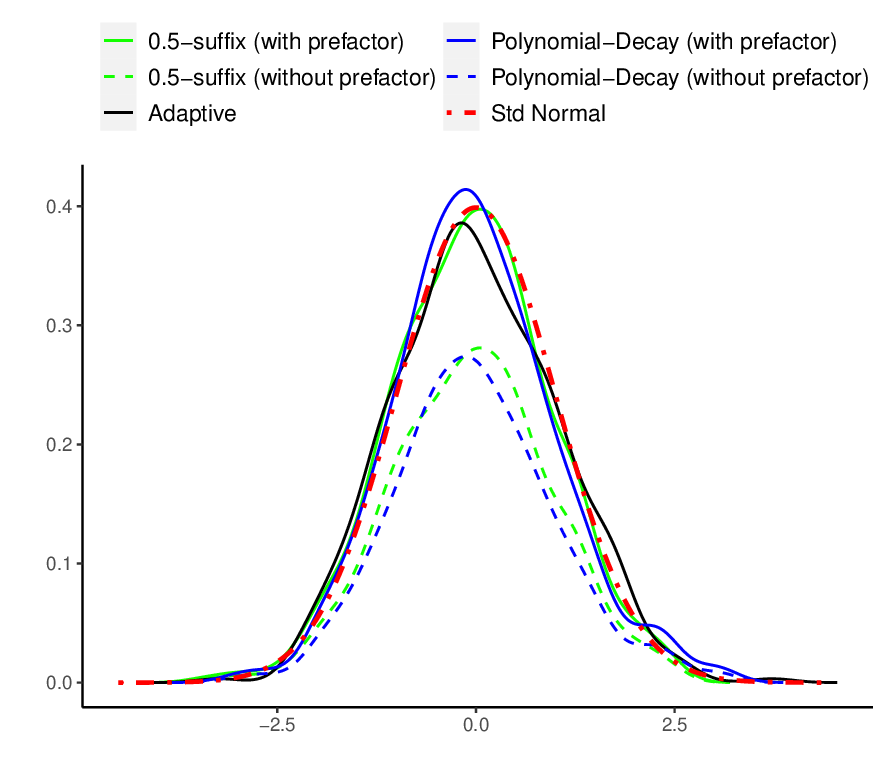}
		} 
		\subfigure[Logit Loss]{\includegraphics[width=0.45\textwidth]{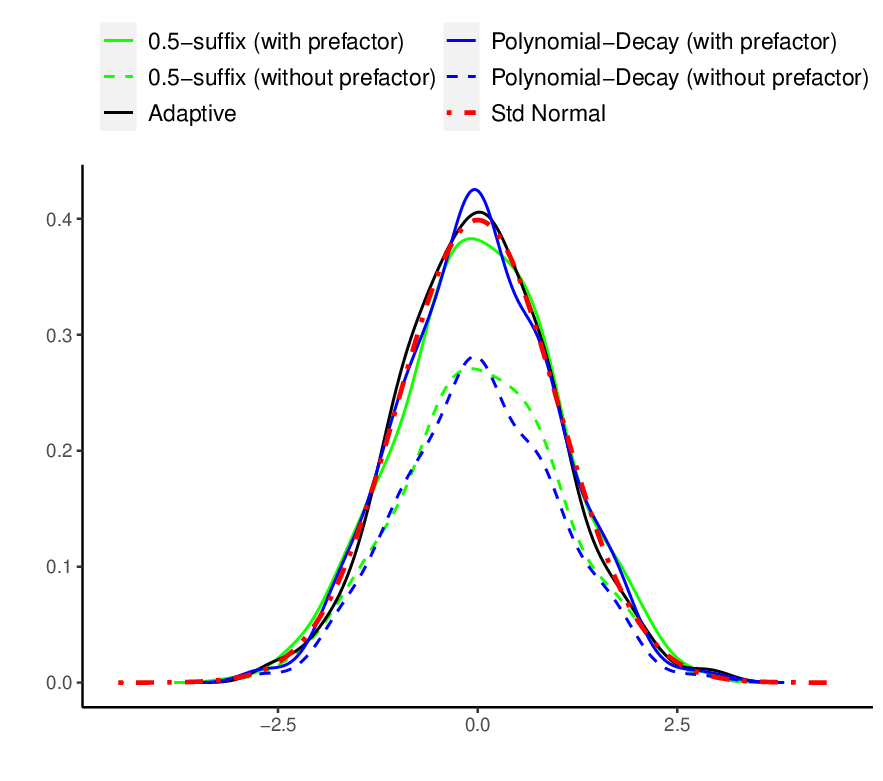}
		}\\
		\caption{Density plot for the standardized error with and without prefactor $w$. Here the number of iterations $n=100000$, and all the measurements are averaged over 450 independent runs. The red line denotes a standard normal distribution}
		\label{fig1} 
	\end{figure}

As shown in Figure \ref{fig1}, the standardized error (scaled with the prefactor $w$) exhibits an approximate standard normal distribution for all three averaging schemes. However, for the polynomial decay and suffix averaging schemes, the standardized error without the prefactor $w$ has a significantly different density compared to the standard normal distribution. These findings support the conclusion stated in Theorem \ref{tm2}, affirming the validity of the asymptotic normality of weighted SGD solutions and the correctness of the limiting covariance matrix.

\subsection{Non-asymptotic performance of different averaging schemes}
In this section, we will show that the adaptive weighted averaging scheme in \eqref{eq:optimal_recursion} has a good non-asymptotic performance in different cases.
	
	\subsubsection{Linear model: optimal MSE}
We first validate the optimality in terms of finite sample MSE in the linear regression model, which is a generalization of the mean estimation model. The linear regression model we examine employs identical simulation settings as described in Section \ref{sec:simu1}. Additionally, we incorporate a specific scenario of the mean estimation model with $a_{i}=1$ and $x^{*}=0$. We present coordinate-averaged MSE at certain steps in Figure \ref{fig:MSE} and include more detailed results in the appendix. We can see that the adaptive weighted averaging outperforms other averaging schemes. When $n$ is large enough ASGD has a similar performance with adaptive weighted SGD. It is also consistent with our conclusion that adaptive weighted SGD and ASGD have the same limiting covariance. For the linear regression model, MSE of the optimal weighted SGD is always smaller than that of ASGD. It indicates that adaptive weighted SGD has the potential to beat ASGD in not only the mean estimation model but also a more general optimization problem.  
	
	\begin{figure}[h]
		\centering
		\subfigure[Linear model]{\includegraphics[width=0.45\textwidth]{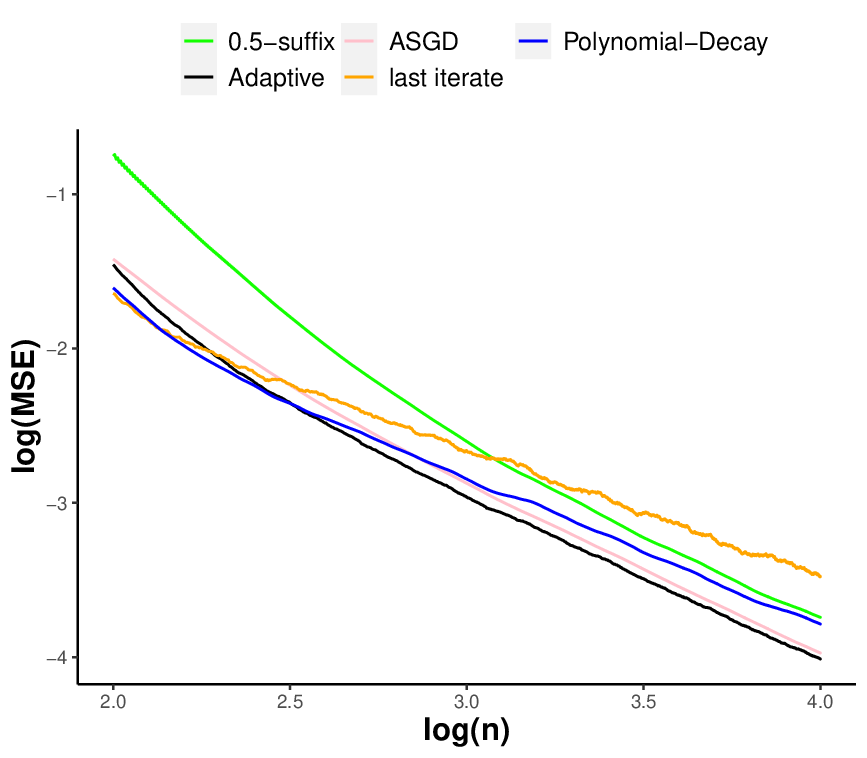} \includegraphics[width=0.45\textwidth]{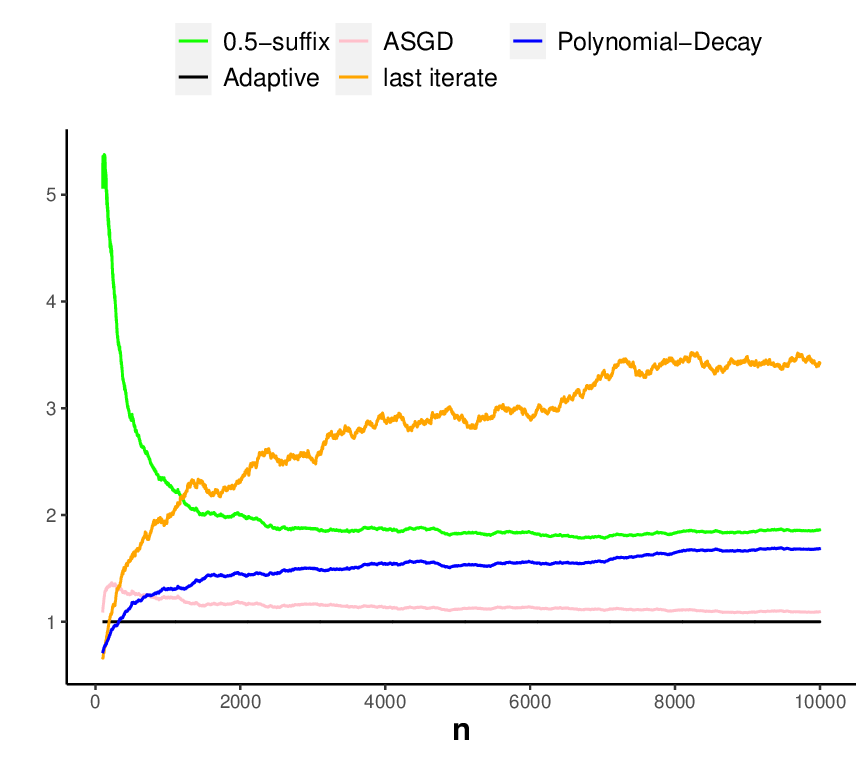}}\\
		\subfigure[Mean estimation model]{\includegraphics[width=0.45\textwidth]{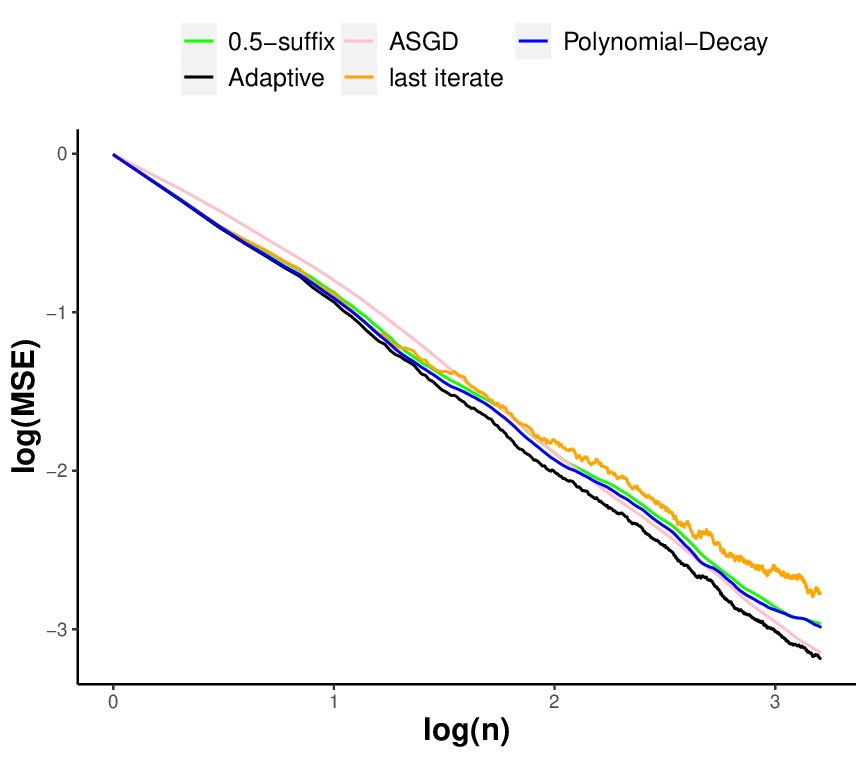} \includegraphics[width=0.45\textwidth]
			{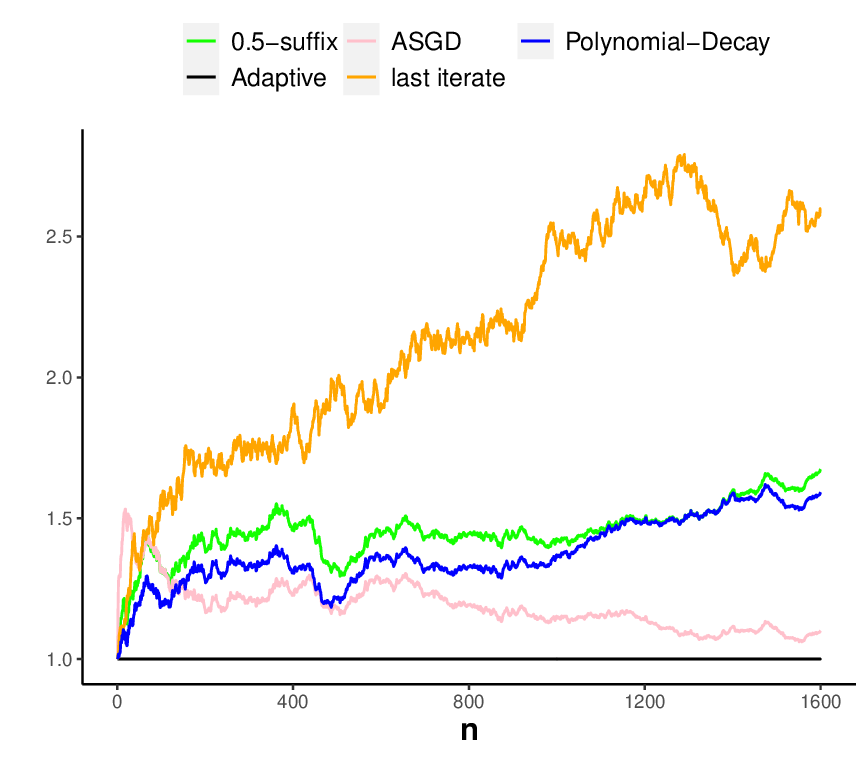}}\\
		\caption{Left: Log-log plots for MSE. 
			Right: the curves stand for the ratio of MSE between different averaging schemes and adaptive weighted averaging at each step. The baseline (black line) is for the adaptive weighted averaging. Here the step size $\eta_i=i^{-0.8}$, and all the measurements are averaged over 400 independent runs. }
		\label{fig:MSE} 
	\end{figure}
	
\subsubsection{Expectile regression: trend of optimal weight}
Expectiles have important applications in finance and risk management. It is closely associated with the commonly adopted measures: Value at Risk (VaR) and Conditional Expected Shortfall (CES). It was proposed by \cite{newey1987asymmetric}, and has been widely used and researched in statistical and economic literature \citep{efron1991regression, taylor2008estimating}. We discuss expectile regression as it has a non-smooth objective function,
	$$F(x) = \E_{y\sim\Pi}(|\rho - 1_{\{y < x\}}|(y - x)^2), 0<\rho<1.$$
For the expectile estimation problem, the optimal weights in Section \ref{sec4} do not have a closed form solution. So we use Monte-Carlo simulation to numerically compute the inverse of the covariance matrix and obtain the oracle weights based on \eqref{matrixform}. We then compare these oracle weights with all the weighting schemes that we studied previously. The weights for each SGD iterate are plotted in Figure \ref{fig:weights} with the total iteration $n=50$. The most remarkable feature of the oracle weight is its \emph{highest} weight assigned to the last iterate. The adaptive weight we proposed in Section \ref{newave} is able to capture this characteristic, while the other averaging schemes fail to recover it. This observation shows that our adaptive weighted averaging scheme is also promising for the non-smooth optimization problem as it aligns the closest with the trend of the oracle weights. 
	
\begin{figure}
\centering	\includegraphics[width=0.6\textwidth]{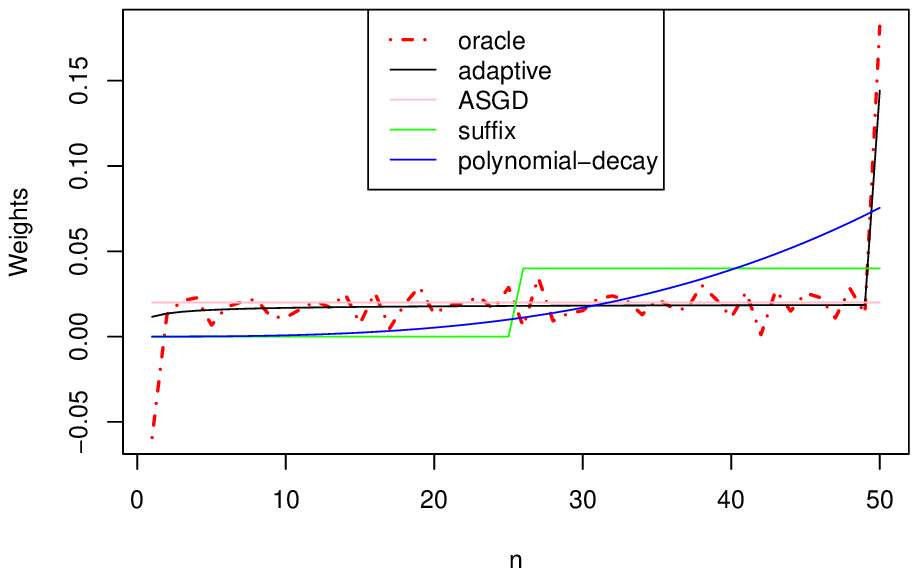}
\caption{Comparison of different weight schemes under expectile regression model with $\rho=0.8$ and the step size $\eta_i = i^{-0.505}$. The oracle weights are numerically computed via Monte-Carlo simulation with 50000 repetitions.} 
\label{fig:weights} 
\end{figure}

\section{Discussion}\label{sec6}
This paper provides optimal conditions for CLT for weighted averaged SGD. We also study CLT for a broad range of weighted averaged SGD solutions, showing that the limiting covariance matrix is equal to the ``sandwich'' form, i.e., the limiting covariance matrix of ASGD, with a prefactor. This constitutes the first asymptotic distribution result for general weighted averaged SGD and holds significant importance for statistical inference.  While applicable to many existing averaging schemes, a challenging future work is to extend the result under weaker assumptions and conditions. We highlight that although certain existing weighted averaged SGD methods exhibit faster convergence than ASGD in the non-asymptotic view or in specific settings without strong assumptions, they may incur a larger variance, suggesting a trade-off. Additionally, we delve into the non-asymptotic mean squared error (MSE) analysis of weighted averaged SGD in the linear regression model and propose a novel averaging scheme—adaptive averaged SGD—which exhibits asymptotic normality with optimal limiting covariance and favorable finite sample MSE.

\bibliography{Bibliography-MM-MC}

\newpage
\appendix
\section{Proof of Main Results}
We first introduce some notation. For positive sequences $a_{n,i}$ and $b_{n,i}$ with $1\leq i \leq n$ (or simply $a_n$ and $b_n$), we say $a_{n,i}\lesssim b_{n,i}$ if there exists a constant $C$ such that $a_{n,i} \leq Cb_{n,i}$ for all $1\leq i \leq n$, and $a_{n,i}\gtrsim b_{n,i}$ if there exists a constant $C$ such that $a_{n,i} \ge Cb_{n,i}$ for all $1\leq i \leq n$, and $a_{n,i}\asymp b_{n,i}$ if $a_{n,i}\lesssim b_{n,i}$ and $a_{n,i}\gtrsim b_{n,i}$.
\subsection{Technical Lemmas}
Before the proof of main results, we present some technical lemmas and defer the proof to section \ref{AL}. The first one is a lemma in \citet{polyak1992acceleration}:
\begin{lemma}\label{polyak}
	Choose the step size as $\eta_i=\eta i^{-\alpha}$ with $\eta>0$ and $0.5<\alpha<1$. For a real symmetric positive definite matrix $A$, define a matrices sequence $Y_i^k$: $Y_i^i={\textbf{I}}$ and for any $k>i$:
	$$Y_i^{k}=\prod_{j=i+1}^k ({\textbf{I}}-\eta_jA).$$
	We also define $\bar{Y}_i^n$ and $\phi^n_i$ as follows,
	$$\bar{Y}_i^n=\eta_i\sum^{n}_{k=i}Y^k_i, \ n\ge i,$$
	$$\phi^n_i=A^{-1}-\bar{Y}^n_i.$$
	Then $\exists \ 0<K<\infty$ such that $\forall \ j$ and $i \ge j$
	$$|\phi^n_i|\leq K,$$
	$$\lim_{n \rightarrow \infty}\frac{1}{n} \sum^{n}_{i=1}|\phi^n_i|=0.$$
\end{lemma}
Lemma~\ref{polyak}~ is a simple reduction of Lemma 1 in \citet{polyak1992acceleration}. The term $Y_i^k$ appears frequently in the explicit form of weighted SGD solutions. The next lemma further investigates the order of some quantities related to it. We defer the proof to section \ref{AL}.

\begin{lemma} \label{lemma Yij}
	Let $\lambda > 0$ be a constant. Define a real sequence $\{Y_{(\lambda)}{}_i^k\}$ with
	\begin{equation}
		Y_{(\lambda)}{}_i^k = \left\{\begin{array}{cc}
			1 & \text{if $i = k$,} \\
			\displaystyle\prod_{j = i+1}^k (1 - \lambda \eta_j) & \text{if $i < k$.}
		\end{array}\right.
	\end{equation}
	Then for all $0 \le i \le k$, we have
	\begin{enumerate}
		\item \label{lemma Yij-1} \begin{equation}
			|Y_{(\lambda)}{}_i^k| \asymp \exp\left\{\frac{\lambda\eta}{1-\alpha} \left(i^{1-\alpha} - k^{1-\alpha} \right)  \right\}
			\le \exp\left\{ - \lambda\eta i^{-\alpha}(k-i) \right\}.
		\end{equation}
		\item \label{lemma Yij-2} For $\beta, \gamma > 0$,
		\begin{equation}
			\sum_{i=1}^k \exp\left(\beta i^{1-\alpha}\right) i^{-\gamma\alpha} \asymp \exp\left(\beta k^{1-\alpha}\right) k^{-(\gamma -1)\alpha},
		\end{equation}
		which implies
		\begin{equation}
			\sum_{i=1}^k |Y_{(\lambda)}{}_i^k|^\beta |i^{-\alpha}|^\gamma
			\asymp k^{-(\gamma - 1)\alpha}.
		\end{equation}
        \item For any $\beta >0$,
        $$\sum^k_{i=j}\exp (-\beta i^{1-\alpha})\lesssim \exp(-\beta j^{1-\alpha})j^{\alpha}.$$
	\end{enumerate}
\end{lemma}
Here and in the sequel, we treat $x_0$ as fixed. We also invoke Lemma 3.1 in \cite{zhu2021online}, which have the following conclusion: for some constant $C$,
$$\| x_n-x^*\| \leq Cn^{-\alpha/2}, \ \ \\| x_n-x^*\|^2 \leq Cn^{-\alpha}.$$

\subsection{Proof of Theorem \ref{tm0}}\label{ptm0}
\begin{proof}

We decompose the SGD iterates as follows,
\begin{equation}\label{sgddecomp}
    x_i-x^*=x_{i-1}-x^*-\eta_i A(x_{i-1}-x^*)-\eta_i  R_{i}-\eta_i D_{i}-\eta_i \nabla f(x^*, \xi_i),
\end{equation}
where 
$$R_{i}=\nabla F(x_{i-1})-A(x_{i-1}-x^*)=\mathcal{O}(|x_{i-1}-x^*|^2),$$
$$D_{i}= \nabla f(x_{i-1}, \xi_i)-\nabla F(x_{i-1})- \nabla f(x^*, \xi_i)$$
due to Taylor's expansion and stochastic Lipshitz property. Notice that $\nabla F(x^*)=0$ here. Denote $\delta_{i} = x_{i} - x^{*}$ the error sequence. The weighted averaged error sequence $\tilde{\delta}_n=\sum_{i=1}^n w_{n,i}\delta_i$ takes the following form, 
\begin{align}\label{wsgddecomp}
\tilde{\delta}_n &= \sum_{i=1}^n w_{n,i}\prod_{j=1}^i ({\bf I}_d-\eta_jA)\delta_0-\sum_{i=1}^n \Phi_{n,i} (R_i+D_i+\nabla f(x^*, \xi_i))\\
&= \eta_1^{-1}({\bf I}_d-\eta_1A)\Phi_{n,1}\delta_0-\sum_{i=1}^n \Phi_{n,i} (R_i+D_i+\nabla f(x^*, \xi_i)).
\end{align}

It is clear that $|   V_n^{-1/2}\eta_1^{-1}({\textbf{I}}_d-\eta_1A)\Phi_{n,1}| \rightarrow 0$ by Condition (\ref{cond1}). Notice that $\{D_i\}$ is a martingale difference sequence. For any $\nu>0$, there exists an integer $K$ such that $2(K+1)^{-\alpha}<\nu$. For any $n >K$, we decompose the martingle difference error term by $K$,

\begin{align*}	 
 \| V_n^{-1/2}\sum_{i=1}^n \Phi_{n,i} D_i\|^2& =  \| V_n^{-1/2}\sum_{i=1}^K \Phi_{n,i} D_i\|^2+ \| V_n^{-1/2}\sum_{i=K+1}^n \Phi_{n,i} D_i\|^2
\\&\leq  \sum_{i=1}^K  |  V_n^{-1/2}\Phi_{n,i} |^2   \| D_i\|^2 +\sum_{i=K+1}^n | V_n^{-1/2}\Phi_{n,i}|^2  \| D_i\|^2
		\\& \lesssim  \sum_{i=1}^K | V_n^{-1/2}\Phi_{n,i} |^2 i^{-\alpha} +\sum_{i=K+1}^n | V_n^{-1/2}\Phi_{n,i}|^2 (K+1)^{-\alpha}\\ 
        & \lesssim  \sum_{i=1}^K | V_n^{-1/2}\Phi_{n,i} |^2 i^{-\alpha}+(K+1)^{-\alpha}\sum_{i=K+1}^n | V_n^{-1/2}\Phi_{n,i}S^{1/2}|^2 .
\end{align*}	 
Since $K$ is fixed, Condition (\ref{cond1}) implies that there exists $N>0$ such that for all $n\ge N$, the first term is smaller than $\nu$. For the second term, notice that
\begin{align*}	
\sum_{i=K+1}^n | V_n^{-1/2}\Phi_{n,i}S^{1/2}|^2 &=\sum_{i=K+1}^n | V_n^{-1/2}\Phi_{n,i}S\Phi_{n,i}V_n^{-1/2}| \\ 
&=\sum_{i=K+1}^n | V_n^{-1}\Phi_{n,i}S\Phi_{n,i}| \leq \textbf{trace}({\textbf{I}}_d)=d
\end{align*}	
because $\sum_{i=1}^n  V_n^{-1}\Phi_{n,i}S\Phi_{n,i}={\textbf{I}}_d $ and the largest eigenvalue of a positive definite matrix is bounded by its trace. Finally, we have
$$ \| V_n^{-1/2}\sum_{i=1}^n \Phi_{n,i} D_i\|^2 \lesssim \nu + d\nu.$$
As a result, the martingale difference error term converges to $0$.

The other error term also vanishes since 

\begin{align*}	 
 \| V_n^{-1/2} \sum_{i=1}^n \Phi_{n,i} R_i\| & \leq  \sum_{i=1}^n  | V_n^{-1/2} \Phi_{n,i} | \| R_i\|
		\\& \lesssim   \sum_{i=1}^n |  V_n^{-1/2}\Phi_{n,i} | i^{-\alpha}\\
        & \leq \sum_{i=1}^K |  V_n^{-1/2}\Phi_{n,i} | i^{-\alpha}+\sum_{i=K+1}^n |  V_n^{-1/2}\Phi_{n,i} | i^{-\alpha}.
\end{align*}

By the Cauchy-Schwarz inequality,
$$(\sum_{i=K+1}^n |  V_n^{-1/2}\Phi_{n,i} | i^{-\alpha})^2 \leq  (\sum_{i=K+1}^n |  V_n^{-1/2}\Phi_{n,i} |^2) \sum_{i=K+1}^ni^{-2\alpha}\lesssim d(K+1)^{1-2\alpha}.$$  
Then we can use the identical argument as before to show that $\mathbb{E} | V_n^{-1/2} \sum_{i=1}^n \Phi_{n,i} R_i| \rightarrow0$.

The main term that does not vanish is $V_n^{-1/2}\sum_{i=1}^n \Phi_{n,i} \nabla f(x^*, \xi_i)$. To show it converges to the standard normal distribution, it suffices to prove that the sequence satisfies the Lindeberg condition, i.e., for any $\nu >0$,
\begin{equation}
    \label{Lind}
 \lim_{n\rightarrow \infty} \sum_{i=1}^n\mathbb{E} \{|V_n^{-1/2}\Phi_{n,i} \nabla f(x^*, \xi_i) |^2 \mathbbm{1}_{|V_n^{-1/2}\Phi_{n,i} \nabla f(x^*, \xi_i) |^2 \ge \nu}  \} =0.   
\end{equation}
By Assumption \ref{as2}, $\| \nabla f(x^*, \xi_i) \|^2 < \infty$. Let $M_n=\max_{1\leq i \leq n} |V_n^{-1/2}\Phi_{n,i} |^2$, we have
\begin{align*}	 
& \ \ \ \ \mathbb{E} \{|V_n^{-1/2}\Phi_{n,i} \nabla f(x^*, \xi_i) |^2 \mathbbm{1}_{|V_n^{-1/2}\Phi_{n,i}|^2 |\nabla f(x^*, \xi_i) |^2 \ge \nu}  \} \\
& \leq  \mathbb{E} \{|V_n^{-1/2}\Phi_{n,i} \nabla f(x^*, \xi_i) |^2 \mathbbm{1}_{M_n |\nabla f(x^*, \xi_i) |^2 \ge \nu}  \} 
		\\& \leq | V_n^{-1/2} \Phi_{n,i}|^2 | \E \{\nabla f(x^*, \xi_i) |^2 \mathbbm{1}_{M_n |\nabla f(x^*, \xi_i) |^2 \ge \nu} \},
\end{align*}	 
and $\E \{\nabla f(x^*, \xi_i) |^2 \mathbbm{1}_{M_n |\nabla f(x^*, \xi_i) |^2 \ge \nu} \}\rightarrow 0$. The reason is that by Condition (\ref{cond1}), the indicators in (\ref{Lind}) go to $0$, and the Dominated convergence theorem. Then it yields
\begin{align*}	 
& \ \ \ \  \lim_{n\rightarrow \infty} \sum_{i=1}^n\mathbb{E} \{|V_n^{-1/2}\Phi_{n,i} \nabla f(x^*, \xi_i) |^2 \mathbbm{1}_{|V_n^{-1/2}\Phi_{n,i} \nabla f(x^*, \xi_i) |^2 \ge \nu}  \} \\
& \leq \sum_{i=1}^n| V_n^{-1/2} \Phi_{n,i}|^2  \lim_{n\rightarrow \infty} \E \{\nabla f(x^*, \xi_i) |^2 \mathbbm{1}_{M_n |\nabla f(x^*, \xi_i) |^2 \ge \nu} \}
		\\&\leq d\lim_{n\rightarrow \infty} \E \{ |\nabla f(x^*, \xi_i) |^2 \mathbbm{1}_{M_n |\nabla f(x^*, \xi_i) |^2 \ge \nu} \}=0.
\end{align*}		
So (\ref{Lind}) holds and the conclusion is proved.

To show that (\ref{cond2}) implies Condition (\ref{cond1}), we construct a counterexample such that CLT does not hold when Condition (\ref{cond1}) fails. Consider 
$$f(x,\xi)= \frac{1}{2}| \xi-A^{\frac{1}{2}}x|^2=\frac{1}{2}x^TAx-\xi^TA^{\frac{1}{2}}x+\frac{1}{2}\xi^T\xi.$$

where $\xi \in \mathbb{R}^d$ is a mean zero random variable with $\mathbb{E}(\xi\xi^T)=A^{-\frac{1}{2}}SA^{-\frac{1}{2}}$. We have $F(x)=(x^TAx+\mathbb{E}\xi^T\xi)/2$ and $x^*=0$. The SGD iterates take the form
\begin{equation}\label{simplesgd}
    x_i = x_{i-1}-(Ax_{i-1}-A^{\frac{1}{2}}\xi_i)=({\textbf{I}}_d-A)x_{i-1}+A^{\frac{1}{2}}\xi_i.
\end{equation}
Let $z_i=A^{\frac{1}{2}}\xi_i$. The standardized weighted averaged SGD is formulated as
$$\tilde{x}_n= \eta_1^{-1}V_n^{-1/2}({\textbf{I}}_d-\eta_1A)\Phi_{n,1}x_0+V_n^{-\frac{1}{2}}\sum_{i=1}^n\Phi_{n,i}z_i,$$
where $x_0$ is fixed. So we only need to consider the distribution of $V_n^{-\frac{1}{2}}\sum_{i=1}^n\Phi_{n,i}z_i$. When Condition (\ref{cond1}) fails, there exists a subsequence $n_k$ such that $$\lim_{k \rightarrow \infty}\max_{1\leq i \leq n_k} |V_{n_k}^{-\frac{1}{2}} \Phi_{n_k,i} | = c>0. $$ 

Without loss of generality, we can assume the original sequence has this limit and the maximum is taken on $i=1$. Since the matrix sequence $V_{n}^{-1/2} \Phi_{n,1}$ is uniformly bounded in the operator norm, it also has a convergence subsequence. So we can further assume that $V_{n}^{-1/2} \Phi_{n,1}$ converges to some matrix $\Gamma$. Since matrix norms are continuous, we have $|\Gamma|=c$. We can choose the unit eigenvector of the largest eigenvalue of $\Gamma$ and denote it as $v$. Since the CLT in condition \ref{cond2} holds, we have 
$$v^TV_{n}^{-\frac{1}{2}}\Phi_{n,1}z_1+\Theta_n \rightarrow \mathcal{N}(0,1), \mbox{ where }
\Theta_n = v^T\sum_{i=2}^n V_{n}^{-\frac{1}{2}} \Phi_{n,i}z_i.
$$
Moreover, $v^TV_{n}^{-\frac{1}{2}}\Phi_{n,1}z_1 \perp\!\!\!\!\perp \Theta_n$ and $v^TV_{n}^{-\frac{1}{2}}\Phi_{n,1}z_1 \rightarrow cv^Tz_1$. Denote the characteristic function of $v^TV_{n}^{-\frac{1}{2}}\Phi_{n,1}z_1$ and $\Theta_n$ as $\phi_n(t)$ and $\psi_n(t)$, we have 
$$ \phi_n(t)\psi_n(t)=e^{\frac{-t^2}{2}},$$
and hence
$$ \psi_n(t) = \frac{e^{\frac{-t^2}{2}}}{\phi_n(t)}\rightarrow \frac{e^{\frac{-t^2}{2}}}{\mathbb{E} e^{itcv^Tz_1}}.$$
The limit of characteristic functions is still a character function, which implies that $\Theta_n$ has a limiting distribution $\Theta_{\infty}$. Finally, we have that the independent sum $cv^Tz_1 + \Theta_{\infty} $ follows a standard Gaussian distribution. By the Lévy-Cramér theorem (cf \cite{chow1997probability}), $cv^Tz_1$ must also be a Gaussian variable as well. This contradicts condition \ref{cond2} which requires CLT holds for non-Gaussian $z_1$. So we have proved the necessity.
\end{proof}

\subsection{Proof of Corollary \ref{CLTlast}}
\begin{proof}
Define the polynomial $p_{n,i}(x)= \eta_i \prod_{j=i+1}^n (1-\eta_jx)$. To prove 
$$\max_{1\leq i \leq n} |V_n^{-\frac{1}{2}} \Phi_{n,i} | \rightarrow 0,$$ where $V_n=\sum_{i=1}^n \Phi_{n,i}S\Phi_{n,i}$ and $\Phi_{n,i} = \eta_i \prod_{j=i+1}^n ({\textbf{I}}_d-\eta_jA)=p_{n,i}(A)$, it suffices to show that 
$$\max_{1\leq i \leq n} \frac{p_{n,i}(\lambda_{\text{min}}(A))^2}{\sum_{i=1}^np_{n,i}(\lambda_{\text{max}}(A))^2} \rightarrow 0,$$ 
because 
$$| \Phi_{n,i}| = \lambda_{\text{max}}(\Phi_{n,i})=p_{n,i}(\lambda_{\text{min}}(A)),$$
and
\[
\lambda_{\text{min}}(V_n) \geq \lambda_{\text{min}}(S) \lambda_{\text{min}}\left(\sum_{i=1}^n \Phi_{n,i}^2\right)\ge \lambda_{\text{min}}(S) \sum_{i=1}^n p_{n,i} (\lambda_{\text{max}}(A))^2,
\]
hence
\[
|V_n^{-\frac{1}{2}}| \leq \frac{1}{\sqrt{\lambda_{\text{min}}(S) \sum_{i=1}^n p_{n,i} (\lambda_{\text{max}}(A))^2}}.
\]

By Lemma \ref{lemma Yij}, we have
\begin{align*} \frac{p_{n,i}(\lambda_{\text{min}}(A))^2}{\sum_{i=1}^np_{n,i}(\lambda_{\text{max}}(A))^2} &\asymp  i^{-2\alpha}n^{\alpha}\exp\left\{\frac{2\lambda_{\text{min}}(A)\eta}{1-\alpha} \left(i^{1-\alpha} - n^{1-\alpha} \right)  \right\}\\
& \leq i^{-2\alpha}n^{\alpha}\exp\left\{ - \lambda_{\text{min}}(A)\eta n^{-\alpha}(n-i) \right\}.
\end{align*}
It is clear that $\max_{1\leq i \leq n}i^{-2\alpha}n^{\alpha}\exp\left\{ - \lambda_{\text{min}}(A)\eta n^{-\alpha}(n-i) \right\}\rightarrow 0$. So the CLT for the last iterate of SGD holds. Then it suffices to show that  $\eta_n^{-1}V_n \rightarrow V$. Notice that 
\begin{align*} 
V_{n+1}&=\eta_{n+1}^2S+\sum_{k=1}^n \Phi_{n,i}({\textbf{I}}_d-\eta_{n+1}A)S\Phi_{n,i}({\textbf{I}}_d-\eta_{n+1}A)\\
&=({\textbf{I}}_d-\eta_{n+1}A)V_n({\textbf{I}}_d-\eta_{n+1}A)+\eta_{n+1}^2S.
\end{align*}
Let $\Gamma_n=\eta_{n}^{-1}V_n$. Plug $\Gamma_n$ and $AV+VA=S$ into the formula above, we have
\begin{align*}&\ \ \ \ \ \Gamma_{n+1}-V\\
&=({\textbf{I}}_d-\eta_{n+1}A)(\Gamma_{n}-V)({\textbf{I}}_d-\eta_{n+1}A)+\eta_{n+1}^2AVA+(\eta_{n+1}^{-1}-\eta_{n}^{-1})({\textbf{I}}_d-\eta_{n+1}A)V_n({\textbf{I}}_d-\eta_{n+1}A).
\end{align*}
For $n$ large enough, the matrix norm of the last term of the left-hand side can be bounded as
\begin{align*}
| (\eta_{n+1}^{-1}-\eta_{n}^{-1})({\textbf{I}}_d-\eta_{n+1}A)V_n({\textbf{I}}_d-\eta_{n+1}A)| &\leq \eta[(n+1)^{\alpha}-n^{\alpha}]|V_n| \\
&\lesssim \eta \alpha n^{\alpha-1}\sum_{i=1}^n| \Phi_{n,i}|^2 | S| \\
&\lesssim  n^{\alpha-1}\sum_{i=1}^n \eta_i^2 \prod_{j=i+1}^n |({\textbf{I}}_d-\eta_jA)|^2 \\
&\lesssim  n^{\alpha-1}\sum_{i=1}^n \eta_i^2 \prod_{j=i+1}^n (1-\eta_j\lambda_{\text{min}}(A))^2 \\
&\asymp n^{\alpha-1}n^{-\alpha}=n^{-1},
\end{align*}
where the last step is from Lemma \ref{lemma Yij}. Notice that $\eta_{n+1}^2 \lesssim n^{-1}$. Hence there exists a universal constant $C$ such that 
$$|\Gamma_{n+1}-V | \leq (1-\eta_{n+1}\lambda_{\text{min}}(A))^2|\Gamma_n-V|+\frac{C}{n}.$$
Recursively updating the inequality we get
\begin{align*} |\Gamma_{n}-V | &\leq \sum_{i=1}^n \frac{C}{i}\prod_{j=i+1}^n (1-\frac{\eta \lambda_{\text{min}}(A)}{j^{\alpha}})^2 + \prod_{j=1}^n (1-\frac{\eta \lambda_{\text{min}}(A)}{j^{\alpha}})^2 |\Gamma_0-V | \\
& \lesssim \sum_{i=1}^n (i^{-\alpha})^{1/\alpha}\prod_{j=i+1}^n (1-\frac{\eta \lambda_{\text{min}}(A)}{j^{\alpha}})^2 + \prod_{j=1}^n (1-\frac{\eta \lambda_{\text{min}}(A)}{j^{\alpha}})^2 \\
& \asymp n^{\alpha-1}+\prod_{j=1}^n (1-\frac{\eta \lambda_{\text{min}}(A)}{j^{\alpha}})^2 \rightarrow0.
\end{align*}
The last step is also from Lemma \ref{lemma Yij}. So we have proved that $\eta_n^{-1}V_n \rightarrow V $. As a result,

$$n^{\frac{\alpha}{2}}(x_n-x^*) \stackrel{D}{\to} \mathcal{N}(0,\eta V),$$

\end{proof}

\subsection{A CLT with general weights and general form of learning rates}
\label{sec:general}
In this section we shall present a CLT with general weights and general form of learning rates. Proposition \ref{xn2} below provides an explicit bound for the MSE of $x_n$.

\begin{proposition}
	\label{xn2}
	Let Assumptions \ref{as1} and \ref{as2} be satisfied. Define
	\begin{equation}\label{vn}
		v_n = (1-2 \eta_n \mu + 2 \eta_n^2 L_2^2) v_{n-1} + 2 \eta_n^2 \sigma^2(x^*), \, n \ge 1, \,
		v_0 = |x_0 - x^*|^2.
	\end{equation}
	Then $\E |x_n - x^*|^2 \le v_n$. In particular, if $\eta_n \asymp n^{-\beta}$, $\beta \in (0, 1)$, then $v_n \lesssim n^{-\beta}$.
\end{proposition}

\begin{proof}
Recall the decomposition of SGD (\ref{sgddecomp}) and take the norm on both sides,

The error sequence can be written as,
		\begin{align*}
  \delta_i  &= \delta_{i-1} - \eta_i \nabla f(x_{i-1}, \xi_i) \cr
	  &= \delta_{i - 1} - \eta_{i} \nabla F(x_{i - 1}) + \eta_{i} \epsilon_i.
		\end{align*}
By Rio's inequality, since $\E [\epsilon_i | \delta_{i - 1} - \eta_{i} \nabla F(x_{i - 1}) ] =0$, we have
$$ \|\delta_{i}\|^{2} \leq \|\delta_{i - 1} - \eta_{i} \nabla F (x_{i - 1})\|^{2} +  \eta_{i}^{2} \|\epsilon_i\|^{2}.$$
		Further by Assumption \ref{as1} and \ref{as2}, we have 
		\begin{align*}
			\|\delta_{i}\|^{2} &\leq \|\delta_{i - 1} - \eta_{i} \nabla F (x_{i - 1})\|^{2} + \eta_{i}^{2} \|\epsilon_i\|^{2}\cr
			&\leq (1 - 2\eta_{i} \mu) \|\delta_{i - 1}\|^{2} + 2  \eta_{i}^{2} (\|\epsilon_i^*\|^{2} + L_{2}^{2} \|\delta_{i - 1}\|^{2})\cr
			&\leq (1 - 2\eta_{i} \mu+2\eta^2_iL^2_2) \|\delta_{i - 1}\|^{2} + 2  \eta_{i}^{2} \sigma^2(x^*).
		\end{align*}
		Then by definition of $v_n$ we have $\E|\delta_n|^2=\E|x_n-x^*|^2 \leq v_n$. If $\eta_i\asymp i^{-\beta}$, we have

        \begin{align*}
			v_i &\lesssim \prod_{k = 1}^{i} (1 - ck^{-\beta} )v_0+ 2 \sigma^2(x^*)  \sum_{j = 1}^{i} j^{-2\beta} \prod_{k = j + 1}^{i} (1 - ck^{-\beta} ) \cr
         &\asymp {Y_{(c)}}_0^i +  \sum_{j = 1}^{i} {Y_{(c)}}_j^i j^{-2\beta} \asymp i^{-\beta}.
		\end{align*}
        The last two steps are from Lemma \ref{lemma Yij}.
\end{proof}

Let $\lambda_1 \le \ldots \le \lambda_d$ be eigenvalues of $A$. For $j \ge 1$ define 
\begin{equation}
	b_j = \lambda_{\max}({\textbf{I}}_d-\eta_jA) = \max_{1 \le l \le d} |1 - \eta_j \lambda_j|
	= \max( |1 - \eta_j \lambda_1|, \, |1 - \eta_j \lambda_d|).
\end{equation}
For small $\eta_j$, we have $b_j = 1 - \eta_j \lambda_1 < 1$, which plays the role of contraction factor. 

\begin{theorem}
\label{tm00}
Consider the SGD \eqref{eq:1}. Let Assumptions \ref{as1} and \ref{as2} be satisfied. Assume that weights $w_{n,i}$ in (\ref{wasgd}) satisfy $\sum_{i=1}^n w_{n,i} = 1$. Let $v_n$ be defined in (\ref{vn}) and $\chi_k = \prod_{i=1}^k b_i$. Assume that
	\begin{equation}\label{J29720}
		\lim_{n\to \infty} \lambda_{\max}( V_n^{-1} ) \sum_{i=1}^n v_{i-1}
		\left( \sum_{k=i}^n |w_{n,k}| \frac{\chi_k }{ \chi_i} \eta_j \right)^2 
		= 0
	\end{equation}
	and
	\begin{equation}\label{J29721}
		\lim_{n\to \infty} \lambda_{\max}( V_n^{-1/2} ) \sum_{j=1}^n |w_{n, j}| 
		\sum_{i=1}^j   \frac{\chi_j}{\chi_i}  \eta_i v_{i-1}
		= 0.
	\end{equation}
	Further assume (\ref{cond1}).
	Then we have the quenched central limit theorem: for any starting point $x_0$,
	\begin{equation*}
		V_n^{-1/2}(\tilde{x}_n-x^*)\stackrel{D}{\to} \mathcal{N}(0,{\textbf{I}}_d).
	\end{equation*}

\end{theorem}
\begin{proof}
Recall the decomposition in \ref{wsgddecomp}. The same argument in \ref{ptm0} can be applied here to show that \\
$V_n^{-1/2}\sum_{i=1}^n \Phi_{n,i} \nabla f(x^*, \xi_i)$ converges to the standard normal distribution. And we still have \\
$|  V_n^{-1/2}\eta_1^{-1}({\textbf{I}}_d-\eta_1A)\Phi_{n,1}| \rightarrow 0$ by Condition (\ref{cond1}). Recall the definition of 
$$\Phi_{n,i} = \eta_i \sum_{k=i}^n w_{n,k} \prod_{j=i+1}^k ({\textbf{I}}_d-\eta_jA).$$
We can bound its operator norm by
$$|\Phi_{n,i}|\leq \eta_i\sum_{k=i}^n |w_{n,k}|\frac{\chi_k}{\chi_j}.$$
As a result,
\begin{align*}	 
\| V_n^{-1/2}\sum_{i=1}^n \Phi_{n,i} D_i\|^2& \leq  \sum_{i=1}^n  |  V_n^{-1/2}\Phi_{n,i} |^2   \| D_i\|^2 \\
&\lesssim \lambda_{\max}(V_n^{-1})\sum_{i=1}^n \eta_i^2\big(\sum_{k=i}^n |w_{n,k}|\frac{\chi_k}{\chi_j}\big)^2v_{i-1}\rightarrow0
\end{align*}	
by condition \ref{J29720}. Moreover,

\begin{align*}	 
 \| V_n^{-1/2} \sum_{i=1}^n \Phi_{n,i} R_i\| & \leq  \sum_{i=1}^n  | V_n^{-1/2} \Phi_{n,i} |  \| R_i\|
		\\& \lesssim   \sum_{i=1}^n \lambda_{\max}(V_n^{-1/2})|  \Phi_{n,i} | v_{i-1}\\
        &\leq \sum_{i=1}^n \lambda_{\max}(V_n^{-1/2})\eta_i\sum_{k=i}^n |w_{n,k}|\frac{\chi_k}{\chi_j}v_{i-1}\\
        &=\lambda_{\max}(V_n^{-1/2})\sum_{k=1}^n|w_{n,k}|\sum_{i=1}^k\frac{\chi_k}{\chi_j} \eta_iv_{i-1} \rightarrow0
\end{align*}	 
by condition \ref{J29721}. So the conclusion follows.
\end{proof}


\subsection{Proof Sketch of Theorem \ref{tm2}}
We present a sketch of the proof of Theorem \ref{tm2}. The error sequence $\delta_{i} = x_{i} - x^{*}$ takes the following form
\begin{align}\label{eq3}
	\delta_i= \delta_{i-1} - \eta_i\nabla F(x_{i-1})+\eta_i \epsilon_i, \ i\ge 1,
\end{align}
where $\epsilon_i= \nabla F(x_{i-1})-\nabla f(x_{i-1},\xi_i)$. Since $\nabla F(x^*)=0$, by Taylor's expansion of $F$ around $x^*$ we have $\nabla F(x_n) \approx A\delta_n$, which inspires the idea to approximate the general SGD sequence with a corresponding linear sequence. 

Consider a linear sequence:
\begin{equation}\label{eq:linear}
	\delta_i'= (I - \eta_i A)\delta_{i-1}' +\eta_i\epsilon_i, \ \delta_0'=\delta_0.
\end{equation}
The following lemma shows the asymptotic normality for the weighted average of the linear sequence $\delta_i'$.
 
\begin{lemma}\label{tm1}
	Let $\delta_i'$ be defined in \eqref{eq:linear}. Then under the settings in Theorem \ref{tm2}, for $\tilde{\delta}_n'= \sum^n_{i=1}w_{n,i}\delta_i'$ we have
	$$\sqrt{n}\tilde{\delta}_n' \stackrel{D}{\to} \mathcal{N}(0,w A^{-1}SA^{-1}),$$
	where $w = \lim_{n \rightarrow \infty}n\sum_{i=1}^{n}(w_{n,i})^{2}$,  $A=\nabla^2F(x^*)$, and  
	$S=\mathbb{E}([\nabla f(x^*,\xi)][\nabla f(x^*,\xi)]^T)$. 
\end{lemma}

To prove Lemma \ref{tm1}, we decompose $\sqrt{n}\tilde{\delta}_n'$ into four terms:

\begin{equation}\label{eq:decompose}
	\begin{aligned}	
\sqrt{n}\tilde{\delta}_n'&=\sqrt{n}\sum^{n}_{i=1}w_{n,i}A^{-1}\epsilon_i+\sqrt{n}\sum^n_{i=1}w_{n,i} Y_0^i\delta_0
		\\&+\sqrt{n}\sum^n_{i=1}w_{n,i} a^n_i\epsilon_i+\sqrt{n}\sum^n_{i=1}b^n_i\epsilon_i,
	\end{aligned}	 
\end{equation}
where $a^n_i=\sum^n_{k=i}(Y^{k}_{i}\eta_i-A^{-1})$, $b^n_i=\sum^n_{k=i+1}(w_{n,k}-w_{n,i})Y^{k}_{i}\eta_i$ and 
$$Y_i^{k}=\prod_{j=i+1}^k ({\textbf{I}}-\eta_jA), k>i, Y_i^i={\textbf{I}}.$$

The last three terms in \eqref{eq:decompose} vanish as $n$ goes to infinity. The first term $\sqrt{n} \sum_{i=1}^n w_{n,i} A^{-1}\epsilon_i$ is a linear combination of martingale differences and the following lemma shows that it is asymptotically normal. 
\begin{lemma}[Martingale difference asymptotic normality]\label{mclt}
	Under the settings in Theorem \ref{tm2},
	$$\sqrt{n} \sum_{i=1}^n w_{n,i} A^{-1}\epsilon_i \stackrel{D}{\to} \mathcal{N}(0,w A^{-1}SA^{-1}).$$
\end{lemma}
Once Lemma \ref{tm1} is established, we can prove Theorem \ref{tm2} using the linear approximation technique below.

\subsection{Proof of Theorem \ref{tm2}}\label{ptm}
\begin{proof}
	Recall the error sequence of SGD iterates $\delta_n=x_n-x^{*}$. It also takes the form $\delta_n=\delta_{n-1}-\eta_n\nabla F(x_{n-1})+\eta_n\epsilon_n$. The weighted averaged error sequence is $\tilde{\delta}_n=\sum_{i=1}^n w_{n,i}\delta_i$. Since $\sum_{i=1}^n w_{n,i}=1$, we have $\tilde{\delta}_n=\tilde{x}_n-x^*$.
	We have also defined the linear error sequence 
	$$\delta'_n=\delta'_{n-1}-\eta_n A\delta'_{n-1}+\eta_n\epsilon_n, \ \delta'_0=x_0-x^*,$$
	$$\tilde{\delta}_n'=\sum^n_{i=1}w_{n,i}\delta'_i.$$
	We claim that Lemma \ref{tm1} is true, i.e., 
	$$\sqrt{n}\tilde{\delta}_n' \stackrel{D}{\to} \mathcal{N}(0,w A^{-1}SA^{-1}),$$
	then it suffices to prove that $\sqrt{n}\tilde{\delta}_n'$ and $\sqrt{n}\tilde{\delta}_n$ are asymptotically equally distributed.
	Let $s_n$ be the difference between the nonlinear and linear sequence. It also takes the following recursion form:
	\begin{equation}
		\begin{aligned}	 
			s_n=\delta_n-\delta'_n&=\delta_{n-1}-\eta_n\nabla F(x_{n-1})-(\textbf{I}-\eta_n A)\delta'_{n-1}
			\\&=({\textbf{I}}-\eta_n A)(\delta_{n-1}-\delta'_{n-1})-\eta_n(\nabla F(x_{n-1})-A\delta_{n-1})
			\\&=({\textbf{I}}-\eta_n A)s_{n-1}-\eta_n(\nabla F(x_{n-1})-A\delta_{n-1}).
		\end{aligned}	 
	\end{equation}
	Recall the definition of $Y^n_i$:
	$$Y_i^{k}=\prod_{j=i+1}^k ({\textbf{I}}-\eta_jA), k>i, Y_i^i={\textbf{I}}.$$
	We can use $Y^n_i$ to rewrite $s_n$ as
	$$s_n=\sum_{i=1}^n Y^n_i\eta_i[A\delta_{i-1}-\nabla F(x_{i-1})].$$
	Define the weighted average difference between the nonlinear and linear sequence:
	$$\tilde{s}_n=\sum^n_{i=1}w_{n,i} s_i=\sum_{i=1}^n w_{n,i} \sum_{j=1}^i Y^i_j\eta_j[A\delta_{j-1}-\nabla F(x_{j-1})].$$
	
	Note that $\sqrt{n}\tilde{\delta}_n = \sqrt{n}\tilde{s}_n+\sqrt{n}\tilde{\delta}_n'$ and $\sqrt{n}\tilde{\delta}_n' \stackrel{D}{\to} \mathcal{N}(0,w A^{-1}SA^{-1})$. To prove  
	$$\sqrt{n}\tilde{\delta}_n=\sqrt{n}(\tilde{x}_n-x^*) \stackrel{D}{\to} \mathcal{N}(0,w A^{-1}SA^{-1}),$$
	it is suffice to prove $\sqrt{n}\tilde{s}_n$ converges to $0$ in probability.
	
	\begin{equation} \label{use}
		\begin{aligned}	
			\| \tilde{s}_n\|_2& \leq \sum_{i=1}^n w_{n,i} \sum_{j=1}^i \|Y^i_j\|_2\eta_j\|A\delta_{j-1}-\nabla F(x_{j-1})\|_2
			\\&\lesssim \frac{1}{n}\sum_{j=1}^n  \|A\delta_{j-1}-\nabla F(x_{j-1})\|_2 \eta_j \Big(\sum_{i=j}^n \|Y^i_j\|_2\Big)
			\\&\lesssim \frac{1}{n}\sum_{j=1}^n  \|A\delta_{j-1}-\nabla F(x_{j-1})\|_2  j^{-\alpha}\big(1+(j+1)^{\alpha}\big)  
			\\&\lesssim \frac{1}{n}\sum_{j=1}^n  \|A\delta_{j-1}-\nabla F(x_{j-1})\|_2 
			\\&\lesssim \frac{1}{n}\sum_{j=1}^n  \|\delta_{j-1}\|_2^2 .
		\end{aligned}	 
	\end{equation}
	The second inequality is obtained by upper bounding $w_{n,i}$ and exchange the order of summations. The third inequality comes from Lemma A.2 in \citet{zhu2021online}. The last inequality is from 
	Taylor's expansion around $x^*$. From Lemma 3.2 in \citet{zhu2021online} we know that 
    $$\mathbb{E}\|\delta_{j-1}\|^2_2 
	\lesssim   (j-1)^{-\alpha}.$$
	So there exists a constant $C>0$ such that
	$$\sum_{j=1}^n\frac{1}{\sqrt{j}}\mathbb{E}\|\delta_{j-1}\|^2_2 \lesssim  \sum_{j=1}^n \frac{1}{\sqrt{j}}(j-1)^{-\alpha}\lesssim\sum_{j=1}^n j^{-0.5-\alpha} \leq C.$$
	
	By Kronecker's lemma, $$\frac{1}{\sqrt{n}}\sum_{j=1}^n\mathbb{E}\|\delta_{j-1}\|^2_2 \rightarrow 0. $$
	As a result, for any fixed $h>0$,
	$$\mathbb{P}(\sqrt{n}\| \tilde{s}_n\|_2>h)\leq \mathbb{P}\Big(\frac{1}{\sqrt{n}}\sum_{j=1}^n  \|\delta_{j-1}\|_2^2>h\Big)\leq \frac{1}{\sqrt{n}h}\mathbb{E}\sum_{j=1}^n\|\delta_{j-1}\|^2_2\rightarrow 0.$$
	
    Thus we proved that
	$\sqrt{n}\tilde{s}_n$ converges to $0$ in probability, and the theorem is proved.

\end{proof}

\subsection{Proof of Lemma \ref{tm1}}
\begin{proof}
	By definition of the linear error term, we have 
	$$\delta'_n=\prod^n_{i=1}({\textbf{I}}-\eta_i A)\delta_0+\sum^n_{i=1}\prod^n_{j=i+1}({\textbf{I}}-\eta_j A)\eta_i\epsilon_i,$$
	where the matrix sequence $Y_i^{k}, \ k\ge i$ is defined as
	$$Y_i^{k}=\prod_{j=i+1}^k ({\textbf{I}}-\eta_jA), k>i, Y_i^i={\textbf{I}}.$$
	Here we also use the convention that $\prod^n_{j=n+1} ({\textbf{I}}-\eta_jA)=I$. Then the weighted averaged error sequence $\tilde{\delta}'_n$ takes the form:
	\begin{equation}\label{f5}
		\begin{aligned}	 
			\tilde{\delta}'_n&=\sum^n_{i=1}w_{n,i}\prod^i_{j=1}({\textbf{I}}-\eta_j A)\delta_0+       \sum^n_{k=1}w_{n,k}\sum^k_{i=1}\prod^k_{j=i+1}({\textbf{I}}-\eta_j A)\eta_i\epsilon_i
			\\&=\sum^n_{i=1}w_{n,i} Y_0^i\delta_0+   \sum^n_{i=1}\sum^n_{k=i}w_{n,k} Y_i^k\eta_i\epsilon_i
			\\&=\sum^n_{i=1}w_{n,i} Y_0^i\delta_0+
			\sum^n_{i=1}w_{n,i}\sum^n_{k=i}Y^{k}_{i}\eta_i\epsilon_i+
			\sum^n_{i=1}\sum^n_{k=i+1}(w_{n,k}-w_{n,i})Y^{k}_{i}\eta_i\epsilon_i
			\\&=\sum^n_{i=1} w_{n,i} A^{-1}\epsilon_i+\sum^n_{i=1}w_{n,i} Y_0^i\delta_0+
			\sum^n_{i=1}w_{n,i}(\sum^n_{k=i}Y^{k}_{i}\eta_i-A^{-1})\epsilon_i+
			\sum^n_{i=1}\sum^n_{k=i+1}(w_{n,k}-w_{n,i})Y^{k}_{i}\eta_i\epsilon_i
			\\&\stackrel{\Delta}{=}I+II+III+IV
		\end{aligned}	 
	\end{equation}
	By Lemma A.2 in \citet{zhu2021online}, 
	$$\sum_{k=i+1}^n \|Y^{k}_{i}\|_2\lesssim (i+1)^{\alpha}.$$
	So we have,
	$$\lim_{n \rightarrow \infty}\|\sqrt{n}\sum^n_{i=1}w_{n,i} Y_0^i\delta_0\|_2\lesssim \lim_{n \rightarrow \infty} \frac{1}{n}\|\sqrt{n}\sum^n_{i=1}Y_0^i\|_2\lesssim \lim_{n \rightarrow \infty}\frac{1}{\sqrt{n}}=0.$$
	Recall $\phi^n_i=\sum^n_{k=i}Y^{k}_{i}\eta_i-A^{-1}$. Then by Lemma \ref{polyak},
 \begin{equation*}
\begin{aligned}	 
	\lim_{n \rightarrow \infty}\E\|\sqrt{n}\sum^n_{i=1}w_{n,i}(\sum^n_{k=i}Y^{k}_{i}\eta_i-A^{-1})\epsilon_i\|^2_2&=\lim_{n \rightarrow \infty}\E\|\sqrt{n}\sum^n_{i=1}w_{n,i}\phi_i^n\epsilon_i\|^2_2 \\
 &\lesssim \frac{1}{n}\sum^n_{i=1}\|\phi_i^n\|^2_2 \\
 &\lesssim \frac{1}{n}\sum^n_{i=1}\|\phi_i^n\|_2 =0.
 \end{aligned}	 
\end{equation*}
	The last inequality is because $\|\phi^n_i\|_2\leq K$. The result shows that $\sqrt{n}II$ and $\sqrt{n}III$ converge to $0$ in $L^2$ norm. For $\sqrt{n}IV$, let $a^n_i=\sum^n_{k=i+1}(w_{n,k}-w_{n,i})Y^{k}_{i}\eta_i$. By Lemma A.2 in \citet{zhu2021online}, we have
	$$\|Y^k_i\|_2\leq \exp (-\lambda\sum_{t=i+1}^k\eta_t),$$
	and
	$$\|a^n_i\|_2 \lesssim \frac{1}{n}\sum_{k=i+1}^n \|Y^k_i\|_2\eta_i \lesssim \frac{1}{n}$$
	We also have
	$$\lim_{n \rightarrow \infty}\sum_{i=1}^n \|a^n_i\|_2 \leq \lim_{n \rightarrow \infty} \sum_{i=1}^n\sum_{k=i+1}^n|w_{n,k}-w_{n,i}|\eta_i \exp (-\lambda\sum_{t=i+1}^k \eta_t).$$
    Now we claim that the term above goes to $0$. 
	As a result,
 \begin{equation*}
\begin{aligned}	\E\|\sqrt{n}IV\|^2_2&=\E\|\sqrt{n}\sum^n_{i=1}a^n_i\epsilon_i\|^2_2 \\
&\lesssim \frac{1}{n}\sum^n_{i=1}\|na_i^n\|^2_2 \\
&\lesssim \frac{1}{n}\sum^n_{i=1}\|na_i^n\|_2=\sum^n_{i=1}\|a_i^n\|_2\rightarrow0.
 \end{aligned}	 
\end{equation*}
	So we only need to show the asymptotic normality of the first term. By Lemma \ref{mclt} we have
	$$\sqrt{n} \sum_{i=1}^n w_{n,i} A^{-1}\epsilon_i\stackrel{D}{\to} \mathcal{N}(0,w A^{-1}SA^{-1}).$$
    Finally, we only need to prove the claim that 
    $$\sum_{i=1}^n\sum_{k=i+1}^n|w_{n,k}-w_{n,i}|\eta_i \exp (-\lambda\sum_{t=i+1}^k \eta_t) \rightarrow 0.$$
    By piecewise Lipshitz condition, we have that except for finite jumping points,
    $$|w_{n,k}-w_{n,k-1}| \leq L_1\frac{1}{n^2}$$
    We first assume that there are no jumping points and the inequality always holds. Then we need the following 3 steps:
	\\
	Step 1: Define $m^i_i=0$ and for any $k>1$,
	$$m_i^k=\sum_{t=i+1}^k \eta_t.$$
	Then our goal is to prove $ \sum_{i=1}^n\sum_{k=i+1}^n|w_{n,k}-w_{n,i}|\eta_i \exp (-\lambda m^k_i) \rightarrow 0$.
	Recall that $\lambda={\min}(\lambda_{\min}(A),1/(2\eta))$.
	Let $\mu=\lfloor 1/\lambda \rfloor+1$. Choose an $N \in \mathbb{N}^+$ such that $\forall k>i\ge N$,
	$$m^k_i \ge \mu \log \frac{k}{i}.$$
	Since $m^k_i \ge \eta(k^{1-\alpha}-i^{1-\alpha})/(1-\alpha)$, we can always find such an $N$.
	\\
	Step 2: Let $b^n_i=\sum_{k=i+1}^n|w_{n,k}-w_{n,i}|\eta_i \exp (-\lambda m^k_i)$. We decompose $ \sum_{i=1}^nb^n_i$ into two parts:
	\begin{equation}
		\begin{aligned}	 
			\sum^n_{i=1}b^n_i&=\sum^n_{i=1} \eta_i \sum^n_{k=i+1}(w_{n,k}-w_{n,i})\exp (-\lambda m^k_i)
			\\&\leq \sum^N_{i=1}\eta_i \sum_{k=i+1}^n|w_{n,k}-w_{n,i}|\exp (-\lambda m^k_i)+\sum^n_{i=N+1}\eta_i\sum^n_{k=i+1}|w_{n,k}-w_{n,i}|\exp (-\lambda m^k_i)
			\\&\stackrel{\Delta}{=}I_1+I_2
		\end{aligned}	 
	\end{equation}
	\\
	Step 3: Show that each term goes to $0$ when $n\rightarrow \infty$.
	For the first term we have
	$$I_1\leq \sum^N_{i=1}\eta_i\frac{2C}{n}\sum_{k=i+1}^n\exp (-\lambda m^k_i)\lesssim\frac{1}{n} \sum^N_{i=1}(i+1)^{\alpha}i^{-\alpha}\lesssim\frac{1}{n}$$
	The second inequality is due to Lemma A.1 and Lemma A.2 in \cite{zhu2021online}. Now there exists a constant $\tilde{C}$ such that $|w_{n,t+1}-w_{n,t}| \leq L_1/n^2$. So for the second term we have
	\begin{equation}
		\begin{aligned}	 
			I_2&=\sum^n_{i=N+1} \eta_i \sum^n_{k=i+1}\{\sum_{t=i}^{k-1}|w_{n,t+1}-w_{n,t}|\}e^{-\lambda m_i^k}
			\\&\lesssim \sum^n_{i=N+1} \eta_i\sum^n_{k=i+1}\sum^{k-1}_{t=i}\frac{1}{n^2}e^{-\lambda m_i^k}
			\\&\lesssim n^{\alpha-2}\sum^n_{i=N+1} \eta_i\sum^n_{k=i+1}\sum^{k-1}_{t=i}\frac{1}{t^{\alpha}}e^{-\lambda m_i^k}
			\\&\lesssim n^{\alpha-2}\sum^n_{i=N+1} \eta_i\sum^n_{k=i+1}m^k_i e^{-\lambda m_i^k}
			\\&=n^{\alpha-2}\sum^n_{i=N+1} \sum^n_{k=i+1}\frac{m^k_i\eta_i}{\eta_k}e^{-\lambda m_i^k}(m^k_i-m^{k-1}_i)
		\end{aligned}	 
	\end{equation}
	The second inequality is because $t^{\alpha}\leq n^{\alpha}$.
	Notice that $\frac{\eta_i}{\eta_k}\leq \frac{k}{i}\leq e^{\frac{m^k_i}{\mu}}$ by step 2,
	\begin{equation}
		\begin{aligned}	 
			I_2&\lesssim n^{\alpha-2}\sum^n_{i=N+1} \sum^n_{k=i+1}m^k_i e^{(\frac{1}{\mu}-\lambda) m_i^k}(m^k_i-m^{k-1}_i)
			\\&\lesssim n^{\alpha-2}\sum^n_{i=N+1}\int^{+\infty}_0me^{(\frac{1}{\mu}-\lambda)m}dm
			\\&\lesssim n^{\alpha-1}\rightarrow 0.
		\end{aligned}	 
	\end{equation}
	We have shown that 
	$$\sum^n_{i=1}b^n_i \lesssim n^{\alpha-1} \rightarrow 0.$$ 
	So the claim is proved under the condition of no jumping points. Then we deal with the finite jumping points case. Without loss of generality, we assume there is one point $0<\kappa<1$ such that the Lipshitz condition holds on $[0,\kappa]$ and $[\kappa,1]$. The case of multiple but finite points follows the same argument.

	Let $d_n=w_{n,\lceil \kappa n \rceil}-w_{n,\lfloor \kappa n \rfloor}$ and $w'_{n,i} = -d_n$ for $i>\kappa n$ and $0$ otherwise. Notice that $|w'_{n,i}| \lesssim 1/n$.We have
    $$|w_{n,k}-w_{n,i}|\leq|w_{n,k}+w'_{n,k}-w_{n,i}-w'_{n,i}|+|w'_{n,k}-w'_{n,i}|,$$ 
    and the term in the claim is a linear combination of $|w_{n,k}-w_{n,i}|$. So we can plug in the triangle inequality where $|w_{n,k}+w'_{n,k}-w_{n,i}-w'_{n,i}|$ has no jumping points, for which $$\lim_{n\rightarrow \infty}\sum_{i=1}^n \eta_i\sum_{k=i+1}^n|w_{n,k}+w'_{n,k}-w_{n,i}-w'_{n,i}|\exp (-\lambda\sum_{t=i+1}^k \eta_t)=0$$
    is proved. So we only need to verify the claim for $|w'_{n,k}-w'_{n,i}|$. 
    By Lemma A.1. in \citet{zhu2021online} we have 
	$$\exp (-\lambda\sum_{t=i+1}^k \eta_t) \leq \exp \bigg(-\frac{\lambda \eta}{1-\alpha}\big(k^{1-\alpha}-(i+1)^{1-\alpha}\big)\bigg).$$ 
    By Lemma \ref{lemma Yij},
\begin{equation}
		\begin{aligned}	 
			& \ \ \ \  \ \ \  \sum_{i=1}^n \eta_i\sum_{k=i+1}^n|w'_{n,k}-w'_{n,i}|\exp (-\lambda\sum_{t=i+1}^k \eta_t)\\
			&\leq \sum^{\lfloor \kappa n \rfloor}_{i=1} \eta_i \sum^n_{k=\lceil \kappa n \rceil} \frac{1}{d_n}\exp \bigg(-\frac{\lambda \eta}{1-\alpha}\big(k^{1-\alpha}-(i+1)^{1-\alpha}\big)\bigg)
			\\&\lesssim \frac{1}{n}\sum^{\lfloor \kappa n \rfloor}_{i=1} i^{-\alpha}\exp\bigg(\frac{\lambda \eta}{1-\alpha}(i+1)^{1-\alpha}\bigg)\sum^n_{k=\lceil \kappa n \rceil}\exp \bigg(-\frac{\lambda \eta}{1-\alpha}k^{1-\alpha}\bigg)
			\\&\lesssim \frac{1}{n}\exp\bigg(\frac{\lambda \eta}{1-\alpha} \lfloor \kappa n \rfloor^{1-\alpha} \bigg)\exp\bigg(-\frac{\lambda \eta}{1-\alpha}\lceil \kappa n \rceil^{1-\alpha}\bigg)\lceil \kappa n \rceil^{\alpha}
			\\&\lesssim \frac{\lceil \kappa n \rceil^{\alpha}}{n}
		\end{aligned}	 
	\end{equation}
So the claim $\lim_{n\rightarrow \infty}\sum_{i=1}^n \eta_i\sum_{k=i+1}^n|w'_{n,k}-w'_{n,i}|\exp (-\lambda\sum_{t=i+1}^k \eta_t)=0$ holds, and we complete the proof of Lemma \ref{tm1}.
\end{proof}

\subsection{Proof of Lemma \ref{mclt}}

\begin{proof}
We further decompose the weighted sum of the martingale difference sequence as
$$ \sqrt{n}\sum^{n}_{i=1}w_{n,i}A^{-1}\epsilon_i= \sqrt{n}\sum^{n}_{i=1}w_{n,i}A^{-1}(\epsilon_i-\nabla f(x^*,\xi_i))+\sqrt{n}\sum^{n}_{i=1}w_{n,i}A^{-1}\nabla f(x^*,\xi_i).$$
By Assumption \ref{as2}, 
$$ \mathbb{E}\|\nabla f(x_{i-1},\xi_i)-\nabla f(x^*,\xi_i)\|^2_2 \leq L^2 \| \delta_{i-1}\|^2$$
By the convexity of \(\| \cdot \|_2^2\) and Jensen's inequality,
\begin{align*}
\| \nabla F(x_{i-1}) - \nabla F(x^*) \|_{2}^2 &= \| \E_{i-1} ( \nabla f(x_{i-1}, \xi_i) - \nabla f(x^*, \xi_i) ) \|_{2}^2  \\
&\leq \E_{i-1}\| \nabla f(x_{i-1}, \xi_i) - \nabla f(x^*, \xi_i) \|_{2}^2 \\
&\leq L^2 \| \delta_{i-1}\|^2.
\end{align*}
Since $\nabla F(x^*)=0$, we have
$$(\E_{i-1}\|\epsilon_i-\nabla f(x^*,\xi_i)\|^2_2)^{\frac{1}{2}} \leq \|\nabla F(x_{i-1}) \|_2 + (\E_{i-1} \| \nabla f(x_{i-1},\xi_i)-\nabla f(x^*,\xi_i)\|^2_2)^{\frac{1}{2}}\leq L \| \delta_{i-1}\|_2.$$
We know that  $\mathbb{E}\|\delta_{i-1}\|^2_2 \lesssim   (i-1)^{-\alpha}$, so
$$ \E\|\epsilon_i-\nabla f(x^*,\xi_i)\|^2_2 \lesssim L^2(i-1)^{-\alpha}.$$
Since $|w_{n,i}| \leq Cn^{-1}$,
\begin{align*}\E \| \sqrt{n}\sum^{n}_{i=1}w_{n,i}A^{-1}(\epsilon_i-\nabla f(x^*,\xi_i))\|_2^2 &\lesssim \frac{1}{n}\sum_{i=1}^n \mathbb{E}\|\epsilon_i-\nabla f(x^*,\xi_i)\|^2_2 \\
&\lesssim \frac{1}{n}(\sum_{i=2}^n (i-1)^{-\alpha}+\| \delta_0\|^2)\asymp n^{-\alpha}
\end{align*}
which vanishes as $n \rightarrow \infty$. Then it suffices to prove the CLT for $\sqrt{n}\sum_{i=1}^n w_{n,i}A^{-1}\nabla f(x^*,\xi_i)$. Let $Z_i=A^{-1}\nabla f(x^*,\xi_i)$ be a sequence of mean-zero variables with finite second moments, and denote the covariance matrix as $\Sigma_n=\sum_{i=1}^n (\sqrt{n}w_{n,i})^2A^{-1}SA^{-1}$. Since $\sum_{i=1}^n (\sqrt{n}w_{n,i})^2 \rightarrow w$ and $|w_{n,i}| \leq Cn^{-1}$, we have 
$$\| \Sigma_n^{-\frac{1}{2}}\sqrt{n}w_{n,i}Z_i \|^2 \leq \frac{4}{w}n w^2_{n,i} \| (A^{-1}SA^{-1})^{-1}\| \|Z_i\|\lesssim \frac{\| Z_i\|^2}{n}. $$
Hence for any $\nu>0$,
\begin{align*}\sum_{i=1}^n\mathbb{E} \{\| \Sigma_n^{-\frac{1}{2}}\sqrt{n}w_{n,i}Z_i \|^2 \mathbbm{1}_{\|\Sigma_n^{-\frac{1}{2}}\sqrt{n}w_{n,i}Z_i \|^2 \ge \nu}  \} &\lesssim \lim_{n\rightarrow \infty} \sum_{i=1}^n \E\{ \frac{\|Z_i\|^2}{n} \mathbbm{1}_{\|Z_i \|^2 \ge n \nu}  \} \\
& = \E\{ \|Z_i\|^2 \mathbbm{1}_{\|Z_i \|^2 \ge n \nu}  \} \rightarrow 0,
\end{align*}
which converges to $0$ because of the fact that $\mathbbm{1}_{\|Z_i \|^2 \ge n \nu} \rightarrow 0$ and Lebesgue's dominated convergence theorem. We have verified the Lindeberg condition. As a result,

$$\Sigma_n^{-\frac{1}{2}}\sqrt{n}\sum_{i=1}^nw_{n,i}Z_i \stackrel{D}{\to}  \mathcal{N}(0,{\textbf{I}}_d).$$
and since $\Sigma_n\rightarrow wA^{-1}SA^{-1}$,
$$\sqrt{n}\sum_{i=1}^nw_{n,i}Z_i \stackrel{D}{\to}  \mathcal{N}(0,wA^{-1}SA^{-1}).$$

Here we also illustrate that $nV_n \rightarrow A^{-1}SA^{-1}$ under the setting of Theorem \ref{tm2}. Recall the SGD iterates
$$x_i = x_{i-1}-(Ax_{i-1}-A^{\frac{1}{2}}\xi_i)=({\bf I}_d-A)x_{i-1}+z_i.$$
where $\E z_i=0$ and $\text{Cov}(z_i)=S$. The weighted averaged SGD is formulated as
$$\tilde{x}_n= \eta_1^{-1}({\bf I}_d-\eta_1A)\Phi_{n,1}x_0+\sum_{i=1}^n\Phi_{n,i}z_i,$$
with its covariance $V_n = \sum_{i=1}^n \Phi_{n,i}S\Phi_{n,i}.$ On the other hand,
\begin{equation}
	\begin{aligned}	\sqrt{n}\tilde{x}_n&=\sqrt{n}\sum^{n}_{i=1}w_{n,i}A^{-1}z_i+\sqrt{n}\sum^n_{i=1}w_{n,i} Y_0^i\delta_0
		\\&+\sqrt{n}\sum^n_{i=1}w_{n,i} a^n_iz_i+\sqrt{n}\sum^n_{i=1}b^n_iz_i,
	\end{aligned}	 
\end{equation}
and we have shown that the last three terms on the left-hand side $L^2$ converge to $0$. Hence
$$\lim_{n\rightarrow \infty}\text{Cov} (\sqrt{n}\tilde{x}_n)= \lim_{n\rightarrow \infty}nV_n=\lim_{n\rightarrow \infty}\text{Cov}(\sqrt{n}\sum^{n}_{i=1}w_{n,i}A^{-1}z_i)= wA^{-1}SA^{-1}.$$
\end{proof}

\subsection{Proof of Theorem \ref{th:fclt}}
\begin{proof}
Recall  (\ref{sgddecomp}) for the decomposition of SGD iterates. Without loss of generality, we can assume $x^*=0$ here.
Let 
$$y_n=y_{n-1}-\eta_nAy_{n-1}-\eta_n\nabla f(x^*,\xi_n), \ \ y_0=0, \ \ S_{L,n}=\sum_{i=1}^ny_i,$$ 
$$\Delta_{D,n}=\Delta_{D,n-1}-\eta_nA\Delta_{D,n-1}-\eta_nD_n, \ \ \Delta_{D,0}=0, \ \ S_{D,n}=\sum_{i=1}^n\Delta_{D,i},$$ 
$$\Delta_{R,n}=\Delta_{R,n-1}-\eta_nA\Delta_{R,n-1}-\eta_nR_n, \ \ \Delta_{R,0}=0, \ \ S_{R,n}=\sum_{i=1}^n\Delta_{R,i}.$$ 
Then the error term due to the starting point goes to $0$ as
$$\frac{1}{\sqrt{n}}\max_{1\leq i \leq n}|S_i-S_{L,i}-S_{D,i}-S_{R,i}| \lesssim \frac{1}{\sqrt{n}}\sum_{i=1}^n |\prod_{j=2}^i ({\bf I}_d-\eta_jA)|)\lesssim \frac{1}{\sqrt{n}}.$$ 
We prove the functional CLT in two steps. First, we prove the function CLT for $S_{L,n}$. To this end, we need to verify the following tightness condition: for any $r_1,r_2>0$, there exists $\delta>0$ and $n_0\in\N^+$ such that for all $n \ge n_0$,
\begin{equation}\label{tight}
    \P\Big(\max_{|i_1-i_2|\leq n\delta}\frac{|S_{L,i_1}-S_{L,i_2}|}{\sqrt{n}}\ge r_1\Big) \leq r_2.
\end{equation}
Without loss of generality, we take $r_1=r_2=r$. Notice that the probability in (\ref{tight}) can be upper bounded by
\begin{equation}
     \sum_{l=0}^{\lceil \frac{1}{\delta} \rceil}\P\Big(\max_{t_l\leq i\leq t_{l+1}}\frac{|S_{L,i}-S_{L,,t_l}|}{\sqrt{n}}\ge \frac{r}{2}\Big),
\end{equation}
which can be further upper bounded by 
\begin{equation}\label{segment}
 \sum_{l=0}^{\lceil \frac{1}{\delta} \rceil} \E \max_{t_l\leq i\leq t_{l+1}}\frac{2^q|S_{L,i}-S_{L,t_l}|^q}{(r\sqrt{n})^q}\leq \frac{2^q}{(r\sqrt{n})^q}\sum_{l=0}^{\lceil \frac{1}{\delta} \rceil}\E \max_{0\leq i\leq \lfloor n\delta\rfloor}|S_{L,t_l+i}-S_{L,t_l}|^q .
\end{equation}
By Markov's inequality, where $t_1=1$, $t_l=t_{l-1}+\lfloor n\delta\rfloor $ for $1< l \leq \lceil 1/\delta \rceil$ and $t_{\lceil1/\delta\rceil+1}=n$.
We apply the powerful maximal inequality Proposition 1 in \cite{Wu2007} to bound the expectation of the maximum as
\begin{align*}
\big(\E\max_{0\leq i\leq \lfloor n\delta\rfloor}|S_{L,t_l+i}-S_{L,t_l}|^q \big)^{1/q}&\leq \big(\E\max_{0\leq i\leq 2^d}|S_{L,t_l+i}-S_{L,t_l}|^q \big)^{1/q}\\
&\leq\sum_{k=0}^d \Big[ \sum_{m=1}^{2^{d-k}}\E|S_{L,t_l+2^km}-S_{L,t_l+2^k(m-1)}|^q\Big]^{1/q} ,
\end{align*} 
where $d=\lceil\log_2(\lfloor n\delta\rfloor)\rceil$. Now we prove the following lemma:
\begin{lemma}$\| y_{m+1}+...+y_{m+j}  \|_q \lesssim \sqrt{j}.$
\end{lemma}
\begin{proof}
When $j\leq m^{\alpha}$, $\| \sum_{i=m+1}^{m+j}y_{i}  \|_q  \leq \sum_{i=m+1}^{m+j}\| y_{i}\|_q\lesssim \sum_{i=m+1}^{m+j} i^{-\alpha/2} \lesssim j(m+1)^{-\alpha/2}\lesssim j^{1-1/2}=\sqrt{j}$. When $j>m^{\alpha}$, the high level idea is to decompose the sum to a linear combination of $Z_k=\nabla f(x^*,\xi_k)$, which is also the sum of martingale differences. The relationship between $y_i$ and $y_j$ can be written as
	\begin{equation}
		y_{i} =  \prod_{k=j+1}^i ({\bf {I}}-\eta_kA)y_j + \sum_{k=j+1}^i  \prod_{l=k+1}^i ({\bf {I}}-\eta_lA) \eta_k\nabla f(x^*,\xi_k). 
		\quad 0 \le j < i,
	\end{equation}
So $y_{m+1}+...+y_{m+j}$ can be expressed as the following linear combination of independent random variables,
$$\sum_{i=m+1}^{m+j}\prod_{k=m+2}^i ({\bf {I}}-\eta_kA)y_{m+1}+\sum_{i=m+2}^{m+j}\prod_{k=m+3}^i({\bf {I}}-\eta_kA)\eta_{m+2}Z_{m+2}+...+\eta_{m+j}Z_{m+j}.$$
We then use Burkholder's inequality since independent mean zero sequences are also martingale differences,
$$ \| y_{m+1}+...+y_{m+j}  \|^2_q \leq (q-1) (\| \psi_{m,j,1}\|_q^2+...+\|\psi_{m,j,j}  \|_q^2).$$
where $\psi_{m,j,1}=\sum_{i=m+1}^{m+j}\prod_{k=m+2}^i ({\bf {I}}-\eta_kA)y_{m+1}$ and 
$$\psi_{m,j,l}= \sum_{i=m+l}^{m+j}\prod_{k=m+l+1}^i({\bf {I}}-\eta_kA)\eta_{m+l}Z_{m+l}, \ 2\leq l \leq j.$$
Recall that $\|y_{m+1}\|_q^2 \lesssim (m+1)^{-\alpha}$, so we have
$$\|\psi_{m,j,1} \|_q^2 \lesssim (m+1)^{-\alpha} \Big(\sum_{i=m+1}^{m+j}\prod_{k=m+2}^i |{\bf {I}}-\eta_kA|\Big)^2 \lesssim (m+1)^\alpha \leq 2j.$$
Since $\|Z_n\|_q^2$ is finite, we also have
$$\|\psi_{m,j,k} \|_q^2 \lesssim (m+l)^{-2\alpha}\Big(\sum_{i=m+l}^{m+j}\prod_{k=m+l+1}^i|{\bf {I}}-\eta_kA|\Big)^2=\mathcal{O}(1).$$
As a result,
$$\| y_{m+1}+...+y_{m+j}  \|^2_q \lesssim 2j+(j-1)\mathcal{O}(1)\lesssim j$$
and the conclusion is proven.
\end{proof}
By the lemma, $\E|S_{L,t_l+2^km}-S_{L,t_l+2^k(m-1)}|^q \lesssim 2^{kq/2}$. Hence there exists a constant $C$ such that
\begin{align*}
    \P\Big(\max_{|i_1-i_2| \leq n\delta}\frac{|S_{L,i_1}-S_{L,i_2}|}{\sqrt{n}}\ge r\Big) &\leq C\frac{2^q}{\delta(r\sqrt{n})^q}(\sum_{k=0}^d ( \sum_{m=1}^{2^{d-k}}2^{kq/2})^{1/q})^q\\
    &=C \frac{2^q}{\delta(r\sqrt{n})^q}2^{qd/2}(\sum_{k=0}^d  (2^{1/q-1/2})^{d-k})^q.
\end{align*}
Since $q>2$, the geometric series $\sum_{k=0}^d  (2^{1/q-1/2})^{d-k}$ converges. We also have $2^d \lesssim n\delta$. So there exists a universal constant $\tilde{C}$ such that 
$$\P\Big(\max_{|i_1-i_2|\leq n\delta}\frac{|S_{L,i_1}-S_{L,i_2}|}{\sqrt{n}}\ge r\Big) \leq  \frac{\tilde{C}(n\delta)^{q/2}}{\delta(r\sqrt{n})^q}=\frac{\tilde{C}\delta^{q/2}}{\delta r^q}.$$
We can choose 
$$\delta \leq \big(\frac{r^{1+q}}{\tilde{C}}\big)^{2/(q-2)}$$ to ensure that 
$$\P\Big(\max_{|i_1-i_2|\leq n\delta}\frac{|S_{L,i_1}-S_{L,i_2}|}{\sqrt{n}}\ge r\Big)\leq r.$$
Hence the tightness condition holds.

Next, we prove the finite-dimensional convergence. In the proof of Lemma \ref{tm1} and Lemma \ref{mclt}, we have shown that 
$$\| \frac{1}{\sqrt{n}}\sum_{i=1}^n y_i - \frac{1}{\sqrt{n}}\sum_{i=1}^n A^{-1}\nabla f(x^*,\xi_i)\| \rightarrow 0.$$
Let $S^*_n=\sum_{i=1}^n\nabla f(x^*,\xi_i)$. For all finite $\{t_j\}_{j=1}^k \in (0,1)$, it entails 
$$  \sum_{j=1}^k\|S^*_{\lfloor nt_j \rfloor}/\sqrt{nt_j}-S_{L,\lfloor nt_j \rfloor}/\sqrt{nt_j}\| \rightarrow0.$$ So the finite-dimensional convergence follows from the Donsker's invariance principle on i.i.d sequence $\nabla f(x^*,\xi_i)$. The tightness condition and the finite-dimensional convergence ensure the FCLT for $S_{L,n}$.
\\

Next, we show that
$$\frac{1}{\sqrt{n}}\| \max_{1\leq i \leq n} |S_{D,i} | \|_2 \rightarrow 0, \ \ \frac{1}{\sqrt{n}}\| \max_{1\leq i \leq n} |S_{R,i} | \|_2 \rightarrow 0.$$
Recall that $\Delta_{D,n} =  \sum_{i=1}^n \eta_i\prod_{j=i+1}^n ({\bf I}_d-\eta_jA)D_i$ where $D_{i}= \nabla f(x_{i-1}, \xi_i)-\nabla F(x_{i-1})- \nabla f(x^*, \xi_i)$.
Again by Proposition 1 in \cite{Wu2007}, for $n\leq 2^d<2n$, we have
\begin{align*}
\|\max_{1\leq i \leq n } |S_{D,i}| \|_2 &\leq \|\max_{1\leq i \leq 2^d} |S_{D,i}| \|_2\\
&\leq \sum_{k=0}^d \Big[ \sum_{m=1}^{2^{d-k}}\|S_{D,2^km}-S_{D,2^k(m-1)}\|_2^2\Big]^{1/2}.
\end{align*} 
Notice that 
$$\|S_{D,2^km}-S_{D,2^k(m-1)}\|_2\leq\sum_{i=1}^{2^k}\| \Delta_{D,2^k(m-1)+i} \|_2\lesssim \sum_{i=1}^{2^k} (2^k(m-1)+i)^{-\alpha}.$$
Here we use the fact that $\Delta_{D,n}$ is the linear combination of a martingale difference sequence, and hence
$$ \|\Delta_{D,n}\|_2^2\leq \sum_{i=1}^n |\eta_i\prod_{j=i+1}^n ({\bf I}_d-\eta_jA)|^2\E |D_i|^2\lesssim \sum_{i=1}^ni^{-3\alpha}\prod_{j=i+1}^n |({\bf I}_d-\eta_jA)|^2\asymp n^{-2\alpha}.$$
Now we can further bound the polynomial series as
$$\sum_{i=1}^{2^k} (2^k(m-1)+i)^{-\alpha} \lesssim (2^km)^{1-\alpha}-(2^k(m-1))^{1-\alpha}.$$
As a result,
\begin{align*}
    \|\max_{1\leq i \leq n } |S_{D,i}| \|_2&\lesssim \sum_{k=0}^d \Big[ \sum_{m=1}^{2^{d-k}} [(2^km)^{1-\alpha}-(2^k(m-1))^{1-\alpha}]^2\Big]^{1/2}\\
 &\leq\sum_{k=0}^d \Big[ \sum_{m=1}^{2^{d-k}} (2^km)^{2-2\alpha}-(2^k(m-1))^{2-2\alpha}\Big]^{1/2}\\
 &=\sum_{k=0}^d\sqrt{(2^k2^{d-k})^{2-2\alpha}}=\sum_{k=0}^d 2^{d(1-\alpha)}.
\end{align*}
Finally, we have
$$ \frac{1}{\sqrt{n}}\|\max_{1\leq i \leq n } |S_{D,i}| \|_2 \lesssim \frac{1}{\sqrt{n}}\sum_{k=1}^d 2^{d(1-\alpha)}\lesssim n^{1-\alpha-1/2}\log_2n=n^{1/2-\alpha}\log_2n \rightarrow 0.$$
Recall that the second term $$\Delta_{R,n} =  \sum_{i=1}^n \eta_i\prod_{j=i+1}^n ({\bf I}_d-\eta_jA)R_i.$$ So
\begin{align*}
    \frac{1}{\sqrt{n}}\| \max_{1\leq i \leq n} |S_{R,i} | \|_2  &\leq\frac{1}{\sqrt{n}}\| \max_{1\leq i \leq n} \sum_{j=1}^i|\Delta_{R,j} | \|_2 \\
    &\leq \frac{1}{\sqrt{n}}\sum_{i=1}^n|\Delta_{R,i} | \\
    & \leq \frac{1}{\sqrt{n}}\sum_{i=1}^n  \sum_{k=1}^i \eta_k\|R_k \| |\prod_{j=k+1}^i ({\bf I}_d-\eta_jA)| \\
    &\lesssim \frac{1}{\sqrt{n}} \sum_{i=1}^n \sum_{k=1}^ii^{-2\alpha}|\prod_{j=k+1}^i ({\bf I}_d-\eta_jA)| \\
    &\lesssim \frac{1}{\sqrt{n}} \sum_{i=1}^n i^{-\alpha}\asymp n^{1/2-\alpha}\rightarrow 0.
\end{align*}
\end{proof}

\subsection{Proof of Corollary \ref{poly}}
\begin{proof}
	Recall the definition of $\theta_{n,i}$:
	\begin{equation*}
		\begin{aligned}	
			\theta_{n,i}  &=\frac{\gamma+1}{\gamma+i} \prod_{j=i+1}^{n} \frac{j-1}{j+\gamma}
			\\&=\frac{\gamma+1}{n}\frac{\Gamma(\gamma+i+1)\Gamma(n+1)}{\Gamma(\gamma+n+1)\Gamma(i+1)}.
		\end{aligned}	 
	\end{equation*}
 
We propose another Lemma here and defer the proof to section \ref{AL}.

\begin{lemma}\label{lemmap}
	The weight $w_{n,i} = \theta_{n,i}$ satisfies $\sum_{i=1}^n w_{n,i}=1$, $w_{n,i} \leq (\gamma+1)/n$ and 
	$$\lim_{n \rightarrow \infty} n\sum_{i=1}^n (w_{n,i})^2=\frac{(\gamma+1)^2}{2\gamma+1}.$$ 
\end{lemma}
Now we only have to verify the smoothness condition. We first show that there exists a constant $\tilde{C}=\gamma(\gamma+1)$ such that for all $1\leq i < n$,
$$|w_{n,i+1}-w_{n,i}| \leq \tilde{C} n^{-2}.$$
Notice for any $n \in \mathbb{N}^+$,
	$$|\theta_{n+1,n+1}-\theta_{n,n}|=\frac{(\gamma+1)(\gamma+n)-n(\gamma+1)}{(\gamma+n+1)(\gamma+n)}=\frac{\gamma(\gamma+1)}{(\gamma+n+1)(\gamma+n)} \leq \frac{\gamma(\gamma+1)}{(n+1)^2},$$
	and
	\begin{equation}
		\begin{aligned}	 
			|\theta_{n+1,n}-\theta_{n+1,n-1}|&=(1-\frac{\gamma+1}{\gamma+n})(\theta_{n,n}-\theta_{n,n-1})
			\\&=\frac{\gamma(\gamma+1)}{(\gamma+n)(\gamma+n-1)}\frac{n}{\gamma+n+1}.
		\end{aligned}	 
	\end{equation}
Since $n\ge 1$, we have $|\theta_{n+1,n}-\theta_{n+1,n-1}|\leq|\theta_{n+1,n+1}-\theta_{n,n}|$. Similarly we can prove that $|\theta_{n+1,i+1}-\theta_{n+1,i}|\leq|\theta_{n+1,n+1}-\theta_{n,n}|$ for any $1\leq i \leq n$. So $|\theta_{n+1,i+1}-\theta_{n+1,i}| \leq \gamma(\gamma+1)/(n+1)^2$ for any $1\leq i \leq n$, or equivalently, $|w_{n,i+1}-w_{n,i}| \leq \tilde{C} n^{-2}$ for all $1\leq i < n$. 

Let $f_n(k/n)=n\theta_{n,i}$, it is clear that this $f_n$ meets the piecewise Lipshitz condition with $\sum_{i=1}^n f_n(k/n) =n$ and $L_1=\gamma(\gamma+1)$. It is actually global Lipshitz with no exception. So all assumptions for Theorem \ref{tm2} are verified and the CLT holds.

\end{proof}

\subsection{Proof of Corollary \ref{Co2}}
\begin{proof}
It is obvious that for $\kappa-$suffix averaging $|w_{n,i}|\lesssim 1/n$, $w=1/\kappa$ and the piecewise Lipshitz condition holds with only one exception. So the CLT holds by Theorem \ref{tm2}.
	
\end{proof}

\subsection{Proof of Proposition \ref{1dlm0}}

We should prove the following Proposition \ref{withinitial}, a more detailed version of Proposition \ref{1dlm0}. Instead of requiring $\eta_1=a_1^{-2}$ in Proposition \ref{1dlm0}, we first consider a general step size. Recall that we have the squared loss function $f(x,\xi_i  = (a_{i}, b_{i}))=(a_i x-b_i)^2/{2}$, and SGD iterates 
\begin{equation}\label{1dlinear}
    x_i=x_{i-1}-\eta_ia_i(a_ix_{i-1}-b_i).
\end{equation}
Here $\eta_i=\eta i^{-\alpha}$ with $0.5<\alpha<1$. 

\begin{proposition}\label{withinitial}
The unique solution to the optimization problem $$
\min_{c = (c_0,\ \cdots,\ c_n): c^{T}\1 = 1}\mathbb{E}\|\sum_{i=0}^n c_i(x_i-x^{*})\|^2$$ 
with $x_i$ defined in (\ref{1dlinear}) is given by
$$c
=\frac{\Theta^TD^{-1}\Theta\mathbbm{1}}{\mathbbm{1}^T\Theta^TD^{-1}\Theta\mathbbm{1}},$$
where $$D=\begin{pmatrix}
		(x_0-x^*)^2&0&\cdots&\cdots&0 \\
		0&\sigma^2a_1^2\eta^2_1&0&\cdots&0\\
		0&0&\sigma^2a_2^2\eta^2_2&\cdots&0\\
		\vdots&\ddots&\ddots&\ddots&\vdots\\
		0&0&0&\cdots&\sigma^2a_n^2\eta^2_n&
\end{pmatrix},$$ 
$$
\Theta = \begin{pmatrix}
		1&0&\cdots&\cdots&0 \\
		\eta_1a^2_1-1&1&0&\cdots&0\\
		0&\eta_2a^2_2-1&1&\cdots&0\\
		\vdots&\ddots&\ddots&\ddots&\vdots\\
		0&0&0&\eta_na^2_n-1&1&.
\end{pmatrix} 
$$
More explicitly,
$$c_{n,0}=\frac{(\dfrac{\sigma}{x_0-x^*})^{2}+a_1^2-\eta^{-1}_1}{S_n},$$
$$c_{n,i}=\frac{\eta_i^{-1}+a_{i+1}^2-\eta_{i+1}^{-1}}{S_n}, 1\leq i \leq n-1, $$
$$c_{n,n}=\frac{1}{\eta_nS_n},$$
where $S_n=(\dfrac{\sigma}{x_0-x^*})^{2}+\sum_{i=1}^n a_i^2.$ 
\end{proposition}

\begin{proof}
    
The SGD error sequence $x_{i}-x^*$ takes the recursion form 
$$x_i-x^*=(1-\eta_ia^2_i)(x_{i-1}-x^*)+\eta_ia_i(b_i-a_ix^*)$$

Let 
$$
\Theta = \begin{pmatrix}
		1&0&\cdots&\cdots&0 \\
		\eta_1a^2_1-1&1&0&\cdots&0\\
		0&\eta_2a^2_2-1&1&\cdots&0\\
		\vdots&\ddots&\ddots&\ddots&\vdots\\
		0&0&0&\eta_na^2_n-1&1&.
\end{pmatrix} 
$$
Then we have
\begin{equation}\label{matrix1d}
\Theta 
\begin{pmatrix}
x_0-x^*\\x_1-x^*\\x_2-x^*\\\vdots\\x_n-x^*
\end{pmatrix}=
\begin{pmatrix}
x_0-x^*\\ \eta_1(b_1-a_1x^*)a_1\\ \eta_2(b_2-a_2x^*)a_2\\\vdots\\\eta_n(b_n-a_nx^*)a_n
\end{pmatrix}.
\end{equation}

We further treat $a_i$ as fixed and denote $\Theta$ as the matrix in the left hand side above. Similar as the mean estimation model, here the optimal weights solution is also determined by $\Sigma = (\mathbb{E}(x_ix_j))_{i,j \ge 0}$, the ``covariance'' matrix of $(x_0,\ x_1,\ x_2,\ \cdots,\ x_n)$. 

We further define 
$$\Phi=\begin{pmatrix}
x_0-x^*\\ \eta_1(b_1-a_1x^*)a_1\\ \eta_2(b_2-a_2x^*)a_2\\\vdots\\\eta_n(b_n-a_nx^*)a_n
\end{pmatrix}, \ X=\begin{pmatrix}
x_0-x^*\\x_1-x^*\\x_2-x^*\\\vdots\\x_n-x^*
\end{pmatrix},$$
and
$$D=\begin{pmatrix}
		(x_0-x^*)^2&0&\cdots&\cdots&0 \\
		0&\sigma^2a_1^2\eta^2_1&0&\cdots&0\\
		0&0&\sigma^2a_2^2\eta^2_2&\cdots&0\\
		\vdots&\ddots&\ddots&\ddots&\vdots\\
		0&0&0&\cdots&\sigma^2a_n^2\eta^2_n&
\end{pmatrix}.$$  

Then $\Theta X=\Phi$ implies that $\E \Theta XX^T \Theta^T=\Theta \E XX^T \Theta^T=\Theta \Sigma \Theta^T=\E \Phi\Phi^T$. By the fact that $\E (x_i-x^*)(x_j-x^*)=0$ for all $i\neq j$, we have $\E \Phi\Phi^T=D$. Thus we have a diagonalization of $\Sigma$ as $\Theta \Sigma \Theta^T=D$.

Using the Lagrangian multiplier method, we can obtain the closed-form solution as
$$\frac{\Sigma^{-1} \mathbbm{1}}{\mathbbm{1}^T\Sigma^{-1} \mathbbm{1}},$$ 
and it remains to show that the solution $$\frac{\Sigma^{-1} \mathbbm{1}}{\mathbbm{1}^T\Sigma^{-1} \mathbbm{1}}=\frac{\Theta^TD^{-1}\Theta\mathbbm{1}}{\mathbbm{1}^T\Theta^TD^{-1}\Theta\mathbbm{1}}$$
is the form of Proposition \ref{withinitial}. Here we give the closed form of the matrix $\Theta^TD^{-1}\Theta$, 
\begin{align*}
&\Theta^TD^{-1}\Theta=\\
&\begin{pmatrix}
		\frac{1}{(x_0-x^*)^2}+\frac{(\eta_1a^2_1-1)^2}{\sigma^2a_1^2\eta_1^2}&\frac{\eta_1a^2_1-1}{\sigma^2a_1^2\eta_1^2}&0&0&\cdots&0 \\
		\frac{\eta_1a^2_1-1}{\sigma^2a_1^2\eta_1^2}&\frac{1}{\sigma^2a_1^2\eta_1^2}+\frac{(\eta_2a^2_2-1)^2}{\sigma^2a_2^2\eta_2^2}&\frac{\eta_2a^2_2-1}{\sigma^2a_2^2\eta_2^2}&0&\cdots&0\\
		0&\frac{\eta_2a^2_2-1}{\sigma^2a_2^2\eta_2^2}&\frac{1}{\sigma^2a_2^2\eta_2^2}+\frac{(\eta_3a^2_3-1)^2}{\sigma^2a_3^2\eta_3^2}&\ddots&\ddots&0\\
		\vdots&\ddots&\ddots&\ddots&\ddots&\frac{\eta_na^2_n-1}{\sigma^2a_n^2\eta_n^2}\\
		0&0&\cdots&0&\frac{\eta_na^2_n-1}{\sigma^2a_n^2\eta_n^2}&\frac{1}{\sigma^2a_n^2\eta_n^2}&
\end{pmatrix},
\end{align*}
and the conclusion can be easily verified.
\end{proof}

The optimal weight of the initialization term $c_{n,0}$ in Proposition \ref{withinitial} depends on $\sigma^2$ and the initial error $x_0-x^*$, both of which can not be observed. To solve this problem, we may consider the two-step estimation: first estimate $\sigma$ and $x^*$ using ASGD or other averaged schemes with a small batch of SGD iterates, then plug them in to obtain the optimal weights. Another approach is to modify the structure of $\Theta$ and equation \ref{matrix1d}. Choosing $\eta_1=a_1^{-2}$, we can exclude $x_0-x^{*}$ from (\ref{matrix1d}) and reduce it to
$$
\begin{pmatrix}
            1&0&\cdots&0\\
		\eta_2a^2_2-1&1&\cdots&0\\
		\ddots&\ddots&\ddots&\vdots\\
		0&0&\eta_na^2_n-1&1&
\end{pmatrix} 
\begin{pmatrix}
x_1-x^*\\x_2-x^*\\\vdots\\x_n-x^*
\end{pmatrix}=
\begin{pmatrix}
 \eta_1(b_1-a_1x^*)a_1\\ \eta_2(b_2-a_2x^*)a_2\\\vdots\\\eta_n(b_n-a_nx^*)a_n
\end{pmatrix}.
$$
In other words, if we plug $\eta_1=a_1^{-2}$ in Proposition \ref{withinitial}, we have $x_1=\eta_ib_ia_i$ and all SGD iterates will be free of $x_0-x^*$. 

Denote $\tilde{\Theta}$ as this reduced matrix in the left hand side above, and we can perform a diagonalization of $\Sigma_{-0}=(\mathbb{E}x_ix_j)_{i,j \ge 1}$, the ``covariance'' matrix of SGD sequence without $x_0$, as follows
$$\tilde{\Theta}\Sigma_{-0} \tilde{\Theta}^T=D_{-0}=\begin{pmatrix}
		\sigma^2a_1^2\eta^2_1&0&\cdots&0\\
		0&\sigma^2a_2^2\eta^2_2&\cdots&0\\
		\vdots&\ddots&\ddots&\vdots\\
		0&0&\cdots&\sigma^2a_n^2\eta^2_n&
\end{pmatrix}.$$
Instead of the optimization problem in proposition \ref{1dlm0}, the decomposition of $\Sigma_{-0}$ enables us to solve a reduced problem excluding the weight on $x_0$
$$
\min_{c = (c_1,\ \cdots,\ c_n): c^{T}\1 = 1}\mathbb{E}\|\sum_{i=1}^n c_i(x_i-x^{*})\|^2,
$$ 
and get the minimizer in the form of  
$$c=\frac{\Sigma_{-0} ^{-1} \mathbbm{1}}{\mathbbm{1}^T\Sigma_{-0} ^{-1}\mathbbm{1}}=\frac{\tilde{\Theta}^TD_{-0} ^{-1}\tilde{\Theta}\mathbbm{1}}{\mathbbm{1}^T\tilde{\Theta}^TD_{-0} ^{-1}\tilde{\Theta}\mathbbm{1}}.$$
The $\sigma^2$ terms in $D_{-0}$ cancel out. Finally, the weighting scheme with $\eta_1=a_1^{-2}$ is 
$$c_{n,i}=\frac{\eta_i^{-1}+a_{i+1}^2-\eta_{i+1}^{-1}}{S_n}, 1\leq i \leq n-1, $$
$$c_{n,n}=\frac{1}{\eta_nS_n},$$
where $S_n=\sum_{i=1}^n a_i^2$.

\subsection{Proof of Corollary \ref{Co3}}
\begin{proof}
	We define $\tilde{x}'_n=\sum_{i=1}^{n-1} c_{n,i} x_i$. From Corollary \ref{CLTlast} we know
    $$n^{\frac{\alpha}{2}}(x_n-x^*) = O_p(1).$$
    Notice that $$\sqrt{n}(\tilde{x}_n-x^*)-\sqrt{n}(\tilde{x}'_n-(1-\frac{1}{n\eta_n})x^*)=\frac{1}{\sqrt{n}\eta_n}(x_n-x^*)=o_p(1)$$
    because
	$$n^{\alpha-\frac{1}{2}}(x_n-x^*)=n^{\frac{\alpha}{2}-\frac{1}{2}}n^{\frac{\alpha}{2}}(x_n-x^*)=o_p(1).$$
	So $\sqrt{n}(\tilde{x}_n-x^*)$ and $\sqrt{n}(\tilde{x}'_n-(1-1/n\eta_n)x^*)=\sqrt{n}\sum_{i=1}^{n-1}c_{n,i}(x_i-x^*)$ has the same asymptotic distribution. Moreover, consider
    $$\sqrt{n}\E |  \sum_{i=1}^{n-1} c_{n,i}(x_i-x^*)- \sum_{i=1}^{n-1}\frac{1}{n}(x_i-x^*)|\leq \sqrt{n}\sum_{i=1}^{n-1}\frac{\eta_i^{-1}-\eta_{i+1}^{-1}}{n}\E |x_i-x^*|\lesssim \frac{1}{\sqrt{n}}\sum_{i=1}^{n-1}i^{\alpha-1}i^{-\alpha/2}\lesssim n^{\alpha/2-1/2}$$
    which vanishes as $n\rightarrow \infty$. So  $\tilde{x}'_n=\sum_{i=1}^{n-1} c_{n,i} x_i$ has the same asymptotic distribution as ASGD. Therefore the conclusion is proved.
\end{proof}

\section{Proof of Auxiliary Lemmas}\label{AL}
\subsection{Proof of Lemma \ref{lemma Yij}}

\begin{proof}
\vphantom{1}
\begin{enumerate}
	\item Clearly it holds for $j = i$. If $j < i$, noticing that $-x \ge \log(1-x)\ge -x-x^2/2$ for all $x \in [0,1/2]$, and the set $\{k: \lambda\eta_k > 1/2\}$ is finite, then we have 
 
\begin{align*}
|Y_{(\lambda)}{}_j^i| & = \prod_{k = j+1}^i |1- \lambda\eta_k|
		\lesssim \exp\bigg( - \lambda \sum_{k=j+1}^i \eta_k \bigg), 
\end{align*}
and
\begin{align*}
\log|Y_{(\lambda)}{}_j^i| & = \sum_{k = j+1}^i \log|1- \lambda\eta_k|
		\gtrsim - \lambda \sum_{k=j+1}^i \eta_k - \lambda^2 \sum_{k=j+1}^i \eta_k^2
\end{align*}
which implies 
 \begin{align*}
|Y_{(\lambda)}{}_j^i| & = \sum_{k = j+1}^i \log|1- \lambda\eta_k|
		\gtrsim \exp\bigg( - \lambda \sum_{k=j+1}^i \eta_k - \lambda^2 \sum_{k=j+1}^i \eta_k^2\bigg) \gtrsim \exp\bigg( - \lambda \sum_{k=j+1}^i \eta_k \bigg)
\end{align*}
since $\sum_{k=j+1}^i \eta_k^2$ is uniformly bounded for all $i$ and $j$. As a result,
	\begin{align*}
		|Y_{(\lambda)}{}_j^i| & = \prod_{k = j+1}^i |1- \lambda\eta_k|
		\asymp \exp\bigg( - \lambda \sum_{k=j+1}^i \eta_k \bigg) \\
		 & = \exp\bigg( - \lambda\eta\sum_{k=j+1}^i k^{-\alpha} \bigg)  
		 \le \exp\left( - \lambda\eta \int_{j+1}^{i} u^{-\alpha} \dif u \right) \\
		 & = \exp\left\{ - \frac{\lambda\eta}{1 - \alpha} \left(i^{1-\alpha} - (j+1)^{1-\alpha} \right)  \right\} \\
		 & = \exp\left\{ \frac{\lambda\eta}{1 - \alpha} \left(j^{1-\alpha} - i^{1-\alpha}\right) + \mathcal{O}(1) \right\}
		  \\
		& \asymp \exp\left\{ \frac{\lambda\eta}{1 - \alpha} \left(j^{1-\alpha} - i^{1-\alpha}\right) \right\}.
	\end{align*}
	Further, by Taylor series,
	\begin{align*}
		& \exp\left\{ \frac{\lambda\eta}{1 - \alpha} (j^{1-\alpha} - i^{1-\alpha}) \right\} \nonumber \\
		= {} & \exp\left\{ -\lambda\eta i^{-\alpha}(i-j) - \frac12 \lambda\eta\alpha u^{-\alpha - 1}(i-j)^2 \right\}
		\tag{for some $u \in (j, i)$} \nonumber \\
		\le {} & \exp\left\{ -\lambda\eta i^{-\alpha}(i-j) \right\}.
	\end{align*}

	\item Since $\exp\left(\beta j^{1-\alpha}\right) j^{-\gamma\alpha}$ is ultimately increasing in $j$, we have
	\begin{align}
		\sum_{j=1}^i \exp\left(\beta j^{1-\alpha}\right) j^{-\gamma\alpha}
		& \asymp \int_{1}^{i+1} \exp\left(\beta t^{1-\alpha}\right) t^{-\gamma\alpha} \dif t
		\nonumber \\
		& \asymp \int_{\beta}^{\beta (i+1)^{1-\alpha}} \e^u u^{-\frac{(\gamma - 1)\alpha}{1-\alpha}} \dif u.  \label{lemma Yij-int}
	\end{align}
	Using integration by parts, we have the following:
	\begin{align*}
		&\int_{\beta}^{\beta(i+1)^{1-\alpha}} \e^u u^{-\frac{(\gamma - 1)\alpha}{1-\alpha}} \dif u \\
		& \asymp  \e^u u^{-\frac{(\gamma - 1)\alpha}{1-\alpha}} \Big|_{\beta}^{\beta(i+1)^{1-\alpha}}
		+ \int_{\beta}^{\beta(i+1)^{1-\alpha}} \e^u u^{-\frac{(\gamma - 1)\alpha}{1-\alpha} - 1} \dif u \\
		& \asymp \exp\left\{ \beta(i+1)^{1-\alpha} \right\} (i+1)^{-(\gamma - 1)\alpha}+ (\beta(i+1)^{1-\alpha}-\beta)\exp\left\{ \beta(i+1)^{1-\alpha} \right\} (i+1)^{-(\gamma - 1)\alpha-(1-\alpha)}\\
		& \asymp \exp\left( \beta i^{1-\alpha} \right) i^{-(\gamma - 1)\alpha}.
	\end{align*}
	Therefore by (\ref{lemma Yij-1}), we have
	\begin{align*}
		\sum_{j=1}^i |Y_{(\lambda)}{}_j^i|^\beta |j^{-\alpha}|^\gamma 
		& \asymp \sum_{j=1}^i \exp\left\{\frac{\lambda\eta \beta}{1-\alpha} \left(j^{1-\alpha} - i^{1-\alpha} \right)  \right\} j^{-\gamma\alpha} \\
		& = \exp\left( - \frac{\lambda\eta \beta}{1-\alpha} i^{1-\alpha} \right)  \sum_{j=1}^i \exp\left(\frac{\lambda\eta \beta}{1-\alpha} j^{1-\alpha}  \right) j^{-\gamma\alpha} \\
		& \asymp  i^{-(\gamma - 1)\alpha}.
	\end{align*}
    \item 
Let $\psi(x)=\exp(-\beta x^{1-\alpha})$. It's a decreasing function of $x$. So we can bound the target term as
\begin{equation}
		\begin{aligned}	 
			\sum^n_{k=j}\exp (-\beta k^{1-\alpha}) &\lesssim \int_{j-1}^{n}\exp (-\beta x^{1-\alpha}))dx
			\\&=\int_{\beta(j-1)^{1-\alpha}}^{\infty} \frac{s^{\frac{\alpha}{1-\alpha}}}{\beta(1-\alpha)}e^{-s}ds\\
			&\lesssim (j-1)^{\alpha}\exp (-\beta (j-1)^{1-\alpha})\\
			&\asymp \exp(-\beta j^{1-\alpha})j^{\alpha}.
		\end{aligned}	 
	\end{equation}
    \end{enumerate}
\end{proof}

\subsection{Proof of Lemma \ref{lemmap}}
\begin{proof}
	Since $\Gamma(\gamma+i+1)\Gamma(n+1)<\Gamma(\gamma+n+1)\Gamma(i+1)$, we have 
	$$|\theta_{n,i}| \leq \frac{\gamma+1}{n}.$$
	From the recursive form of polynomial decay averaged SGD, it is easy to see $\sum_{i=1}^n \theta_{n,i}=1$. The last step is to prove the limitation holds. Define
	$$\Gamma_n(x)=\int^n_0t^{x-1}(1-\frac{t}{n})^n dt=\frac{n^xn!}{z(z+1)(z+2)\cdot\cdot\cdot(z+n)},$$
	where the last equality is from integration by parts. It's well known that $(1-\frac{t}{n})^n\leq (1-\frac{t}{n+1})^{n+1}$ and $\lim_{n\rightarrow \infty}(1-\frac{t}{n})^n = e^{-t}$ for any $t$. So $\Gamma_n(x)\leq \Gamma(x)$. Meanwhile we have an equivalent definition of $\Gamma(x)$:
	
	$$\Gamma(x)=\lim_{n\rightarrow \infty}\frac{n^xn!}{x(x+1)(x+2)\cdot\cdot\cdot(x+n)}=\lim_{n\rightarrow \infty}\Gamma_n(x).$$
	
	So for any $\tau>0$, there exists an $N>0$ such that for all $ n\geq N$, 
	$$0 \leq \Gamma(\gamma)-\frac{n^{\gamma}n!}{\gamma(\gamma+1)(\gamma+2)\cdot\cdot\cdot(\gamma+n)} \leq \tau.$$
	
	As a result, for $n\ge i \ge N$, we have $0\leq \frac{\Gamma(\gamma+i+1)}{\Gamma(i+1)i^{\gamma}}-1 \leq \Gamma(\gamma)\tau$ and $0\leq \frac{\Gamma(\gamma+n+1)}{\Gamma(n+1)i^{\gamma}}-1 \leq \Gamma(\gamma)\tau$, which implies
	
	$$\bigg| \frac{\Gamma(\gamma+n+1)}{\Gamma(n+1)i^{\gamma}}-\frac{\Gamma(\gamma+i+1)}{\Gamma(i+1)i^{\gamma}}\bigg| \leq 2\Gamma(\gamma)\tau.$$
	Furthermore we have 
	\begin{equation}\label{f1}
		\begin{aligned}	 
			\bigg| \frac{i^{\gamma}}{n^{\gamma}}-\frac{\Gamma(\gamma+i+1)\Gamma(n+1)}{\Gamma(\gamma+n+1)\Gamma(i+1)}\bigg|&=\bigg|\frac{t^{\gamma}\Gamma(n+1)}{\Gamma(\gamma+n+1)}  \bigg|\bigg| \frac{\Gamma(\gamma+n+1)}{\Gamma(n+1)i^{\gamma}}-\frac{\Gamma(\gamma+i+1)}{\Gamma(i+1)i^{\gamma}}\bigg| 
			\\ &\leq  2\Gamma(\gamma)\tau \bigg|\frac{n^{\gamma}\Gamma(n+1)}{\Gamma(\gamma+n+1)}  \bigg|
			\\ &\leq 2\Gamma(\gamma)\tau.
		\end{aligned}	 
	\end{equation}
	Now we estimate the following summation
	\begin{equation}\label{f2}
		\begin{aligned}	 
			&\frac{1}{n}\sum^n_{i=1}[(\gamma+1)(\frac{i}{n})^{\gamma}-n(\theta_{n,i})]
			\\ \leq&\frac{1}{n}\sum^N_{i=1}|(\gamma+1)(\frac{i}{n})^{\gamma}|+\sum^N_{i=1}|(\theta_{n,i})|+\frac{\gamma+1}{n}\sum^n_{i=N+1}\bigg|(\frac{i}{n})^{\gamma}-\frac{\Gamma(\gamma+i+1)\Gamma(n+1)}{\Gamma(\gamma+n+1)\Gamma(i+1)}\bigg|
			\\ \leq& \frac{N(\gamma+1)}{n}+\frac{N}{n}\theta_{n,N}+\frac{n-N}{n}\tau(\gamma+1)
			\\ \leq& (\gamma+1)(\frac{N}{2n}+\tau).
		\end{aligned}	 
	\end{equation}
	Let $\tau \rightarrow 0$ and $n \rightarrow \infty$,
	\begin{equation}\label{f3}
		\lim_{n\rightarrow\infty}\frac{1}{n}\sum^n_{i=1}[(\gamma+1)(\frac{i}{n})^{\gamma}-n(\theta_{n,i})]=0.
	\end{equation}
	Since 
	\begin{equation}
		\begin{aligned}	 
			[(\gamma+1)^2(\frac{i}{n})^{2\gamma}-n^2(\theta_{n,i})^2]&=[(\gamma+1)(\frac{i}{n})^{\gamma}-n(\theta_{n,i})][(\gamma+1)(\frac{i}{n})^{\gamma}+n(\theta_{n,i})]
			\\&\leq 2(\gamma+1)[(\gamma+1)(\frac{i}{n})^{\gamma}-n(\theta_{n,i})],
		\end{aligned}	 
	\end{equation}
	we have
	\begin{equation}\label{f4}
		\lim_{n\rightarrow\infty}\frac{1}{n}\sum^n_{i=1}[(\gamma+1)^2(\frac{i}{n})^{2\gamma}-n^2(\theta_{n,i})^2]\leq 2(\gamma+1)\lim_{n\rightarrow\infty} \frac{1}{n}\sum^n_{i=1} [(\gamma+1)(\frac{i}{n})^{\gamma}-n(\theta_{n,i})]=0.
	\end{equation}
	Finally, notice that $$\lim_{n\rightarrow\infty}\frac{1}{n}\sum^n_{i=1}(\gamma+1)^2(\frac{i}{n})^{2\gamma}=\int^1_0 (\gamma+1)^2x^{2\gamma}dx=\frac{(\gamma+1)^2}{2\gamma+1},$$
	We have proved the limitation in Theorem \ref{tm2} holds with 
	$$\lim_{n\rightarrow\infty} n\sum^n_{i=1} \theta_{n,i}^2=\frac{(\gamma+1)^2}{2\gamma+1}.$$
	
\end{proof}

\section{Supplemental Information of Simulation Studies}


The choice of $\alpha=0.505$ follows the work of \citep{zhu2021online}. A similar choice can be found in \citep{chen2020statistical} to make $\alpha$ slightly larger than $0.5$. To verify our result when $\alpha$ is near $1$, we also set $\alpha=0.8$. The choices of parameters of polynomial decay averaging ($\gamma=3$) and suffix averaging ($\kappa=0.5$) are the same as the settings in \citep{shamir2013stochastic} and \citep{rakhlin2011making}.

We report MSE (Table \ref{tb1} and \ref{tb2}) and standard deviations of squared errors (Table \ref{tb3} and \ref{tb4}) in the simulation of the mean estimation model and the squared loss model. As is shown in Tables \ref{tb3} and \ref{tb4}, the standard deviation of optimal weighted averaging is also the smallest among all averaging schemes.

\begin{table}[ht]
\caption{MSE for mean estimation model.}
\vskip 0.15in
\begin{center}
\begin{small}
\begin{sc}
\begin{tabular}{lcccr}
\toprule
& & $\alpha=0.505$ & & \\
\midrule
	Average Scheme      & $n=400$  &  $ n=800$ &  $n=1200$    & $n= 1600$ \\
\midrule
	Optimal Weight    & 0.00251625 &    0.00116611       & 0.00080357 &0.00064877  \\
			ASGD    &0.00254422 &    0.00118549 &	0.00080872&	0.00066366  \\
			$\kappa$-suffix   &  0.00450813&0.00232837&	0.00170125&	0.00143947  \\ 
			Polynomial-Decay    &  0.00464877	&0.00192599	&0.00150453	&0.00010830 \\
			Last Iterate    &  0.02771139	&0.01643044	&0.01495235	&0.01226139 \\
			\midrule
			& & $\alpha=0.8$ & & \\
			\midrule
			Average Scheme      & $n=400$  &  $ n=800$         &  $n=1200$    & $n= 1600$\\\hline 
			Optimal Weight    & 0.00251625 &    0.00116611      & 0.00080357  &0.00064877  \\
			ASGD    &0.00308790 &    0.00139040&	0.00091637 &0.00071408  \\
			$\kappa$-suffix   & 0.00366752&	 0.00168517&	0.00119971 &	0.00108589  \\ 
			Polynomial-Decay    & .00328634	&0.00154893	&0.00119466	&0.00103280 \\
			Last Iterate    & 0.00472179	&0.00250589	&0.00214805	&0.00168590 \\
\bottomrule
\end{tabular}
\label{tb1}
\end{sc}
\end{small}
\end{center}
\end{table}

\begin{table}[ht]
\caption{MSE for squared loss model. Here $\alpha=0.8$.}
\vskip 0.15in
\begin{center}
\begin{small}
\begin{sc}
\begin{tabular}{lccccr}
\toprule
& & squared loss & & \\
\midrule
Average Scheme      & $n=2000$  &  $ n=4000$         &  $n=6000$    & $n= 8000$& $n= 10000$\\
\midrule
				    Optimal Weight    & 0.00052560 &    0.00025011& 0.00016126& 0.00012099 &0.00009732\\
				   ASGD    &0.00061791& 0.00028463& 0.00018354& 0.00013391 & 0.00010643 \\
				   $\kappa$-suffix   & 0.00104944 & 0.00046774& 0.00029666& 0.00022174  &0.00018105\\ 
				    Polynomial-Decay    & 0.00075921& 0.00038657&  0.00025183& 0.00019944&0.00016392 \\
				    Last Iterate    &  0.00121235& 0.00071972& 0.00046712& 0.00041441&0.00033334 \\
\bottomrule
\label{tb2}
\end{tabular}
\end{sc}
\end{small}
\end{center}
\vskip -0.1in
\end{table}

\begin{table}[ht]
\caption{Standard deviations of squared errors for mean estimation model.}
\label{tb3}
\vskip 0.15in
\begin{center}
\begin{small}
\begin{sc}
\begin{tabular}{lcccr}
\toprule
& & $\alpha=0.505$ & & \\
\midrule
	Average Scheme      & $n=400$  &  $ n=800$ &  $n=1200$    & $n= 1600$ \\
\midrule
	Optimal Weight    & 0.00378955 & 0.00179229 & 0.00121118 & 0.00103080  \\
			ASGD    &0.00387207 & 0.00185897 &0.00119596 &0.00106196  \\
			$\kappa$-suffix   &  0.00592971 &0.00372515 &0.00247292& 0.00227080  \\ 
			Polynomial-Decay    &0.00609187 &0.00360348 &0.00289309& 0.00231915 \\
			Last Iterate    &  0.04230295 &0.02266438 &0.02185713&0.01771681 \\
			\midrule
			& & $\alpha=0.8$ & &\\
			\midrule
			Average Scheme      & $n=400$  &  $ n=800$         &  $n=1200$    & $n= 1600$\\\hline 
			Optimal Weight    & 0.00378955 &    0.00179229      & 0.00121118  &0.00103080 \\
			ASGD    &0.00486709 &    0.00217177&	0.00139634
   &0.00108259  \\
			$\kappa$-suffix   & 0.00538974& 0.00313332& 0.00169102& 0.00169910 \\ 
			Polynomial-Decay    & 0.00465715& 0.00276466& 0.00164636& 0.00166781 \\
			Last Iterate    &0.00672886& 0.00371866& 0.00338224 &0.00253484 \\
\bottomrule
\end{tabular}
\end{sc}
\end{small}
\end{center}
\vskip -0.1in
\end{table}

\begin{table}[ht]
\caption{Standard deviations of squared errors for squared loss model. Here $\alpha=0.8$.}
\label{tb4}
\vskip 0.15in
\begin{center}
\begin{small}
\begin{sc}
\begin{tabular}{lccccr}
\toprule
& & squared loss & & \\
\midrule
Average Scheme      & $n=2000$  &  $ n=4000$         &  $n=6000$    & $n= 8000$& $n= 10000$\\
\midrule
				    Optimal Weight    & 0.00033900&    0.00015772& 0.00010264& 0.00007903&0.00006308 \\
				   ASGD    &0.00046388& 0.00018273& 0.00011797& 0.00008705&0.00007048  \\
				   $\kappa$-suffix   & 0.00088597& 0.00035209&  0.00019249& 0.00014681 &0.00012083 \\ 
				    Polynomial-Decay    & 0.00047245& 0.00025178& 0.00015650& 0.00013407 &0.00010312\\
				    Last Iterate    &  0.00069180& 0.00046476& 0.00031687& 0.00025734&0.00019574\\
\bottomrule
\end{tabular}
\end{sc}
\end{small}
\end{center}
\vskip -0.1in
\end{table}
\end{document}